\documentclass[10pt]{extarticle}
\usepackage[a4paper, total={6in, 9in}]{geometry}
\usepackage[utf8]{inputenc}
\usepackage{amsmath}
\usepackage{amsfonts}
\usepackage{amsthm}
\usepackage{xcolor}
\usepackage{graphicx}
\usepackage[sort&compress,numbers]{natbib}
\usepackage{hyperref}
\usepackage[inline]{enumitem}
\usepackage{wrapfig}

\usepackage{xspace}
\newcommand{\eg}{\textit{e.g.}\@\xspace}
\newcommand{\ie}{\textit{i.e.}\@\xspace}
\newcommand{\iid}{i.i.d.\@\xspace}

\newcommand{\mc}{\mathcal}
\newcommand{\R}{\mathbb{R}}
\newcommand{\NN}{\mathbb{N}}
\newcommand{\EE}{\mathbb{E}}
\newcommand{\prob}{\mathbb{P}}
\newcommand{\dint}{\;{\rm d}}
\newcommand{\var}{{\rm Var}}

\newcommand{\dzvector}[1]{\boldsymbol{\mathbf{#1}}}
\renewcommand{\vec}[1]{\dzvector{#1}}

\newcommand{\sign}{{\rm sgn}}

\definecolor{thmcol}{HTML}{224D7F}
\newtheoremstyle%
 {redthm}%
 {}{}%
 {\itshape}
 {}%
 {\color{thmcol}\large\bfseries}%
 {\color{thmcol}.}%
 { }{}

\newtheorem{theorem}{Theorem}

\newtheorem{lemma}{Lemma}

\newtheorem{remark}{Remark}
\newenvironment{enumproof}
  { \begin{enumerate}[label={\textbf{\arabic*)}}, wide, labelwidth=!, labelindent=0pt] }
  { \end{enumerate} }

\newcommand{\x}{{\vec x}}
\newcommand{\X}{{\vec X}}

\newcommand{\W}{{\vec W}}

\newcommand{\st}{\;\big|\;}

\newcommand{\spa}{{\rm sp}}
\newcommand{\tim}{{\rm tm}}

\usepackage{booktabs}
\usepackage{multirow}
\usepackage{caption}
\usepackage{subcaption}

\title{AZ-whiteness test: a test for uncorrelated noise on spatio-temporal graphs}

\usepackage{authblk}
\author[1]{Daniele Zambon\thanks{Corresponding author, \texttt{daniele.zambon@usi.ch} .}}
\author[1,2]{Cesare Alippi}
\affil[1]{The Swiss AI Lab IDSIA \& Universit\`a della Svizzera italiana, Switzerland.}
\affil[2]{Politecnico di Milano, Italy.}

\begin{document}

\maketitle

\begin{abstract}
\noindent
We present the first whiteness test for graphs, i.e., a whiteness test for multivariate time series associated with the nodes of a dynamic graph. The statistical test aims at finding serial dependencies among close-in-time observations, as well as spatial dependencies among neighboring observations given the underlying graph.
The proposed test is a spatio-temporal extension of traditional tests from the system identification literature and finds applications in similar, yet more general, application scenarios involving graph signals.
The AZ-test is versatile, allowing the underlying graph to be dynamic, changing in topology and set of nodes, and weighted, thus accounting for connections of different strength, as is the case in many application scenarios like transportation networks and sensor grids. 
The asymptotic distribution --- as the number of graph edges or temporal observations increases --- is known, and does not assume identically distributed data.
We validate the practical value of the test on both synthetic and real-world problems, and show how the test can be employed to assess the quality of spatio-temporal forecasting models by analyzing the prediction residuals appended to the graphs stream.
\end{abstract}

\section{Introduction}

In recent years, machine learning methods based on graph-structured data have made significant advances in the field of multivariate time-series analysis and resulted in major achievements, \eg, in forecasting performance and missing values imputation \cite{wu2019graph,cini2021filling}.
In this paper, we address the problem of testing whether given time series can be considered white noise or not and, by that, we aim at assessing the optimality of machine learning models trained to solve associated forecasting tasks.

We focus on spatio-temporal time series where the spatial domain is defined by a graph $G=(V,E)$ and stochastic time series (or {node signals}) $\x_v[t] \in \R^F$, with $t=1,2,\dots,$ are associated with the nodes $v\in V$ of $G$.
That said, the framework can be extended to deal with graphs seen as a realization of a random variable.
Graph $G$ can be static, meaning that is constant over time, or dynamic, hence modeling frameworks where topology and number of nodes can change. 
Dynamic graphs appear frequently in cyber-physical systems where sensors can be added or removed, or node data is missing for certain lapses of time following communication problems or faults. Another example is provided by social networks where users expand their set of friends or change preferences and interests.  
It is also extremely important to consider graphs with weighted edges encoding, for instance, the capacity or the strength of the link.
Figure~\ref{fig:graph-signal} provides a visualization of a dynamic graph $G$ with associated graph signal $\X$, that is, the collection of all node signals $\x_v[t]$, for nodes $v$ and time steps $t$.

\begin{figure}
    \centering
    \includegraphics[width=\textwidth]{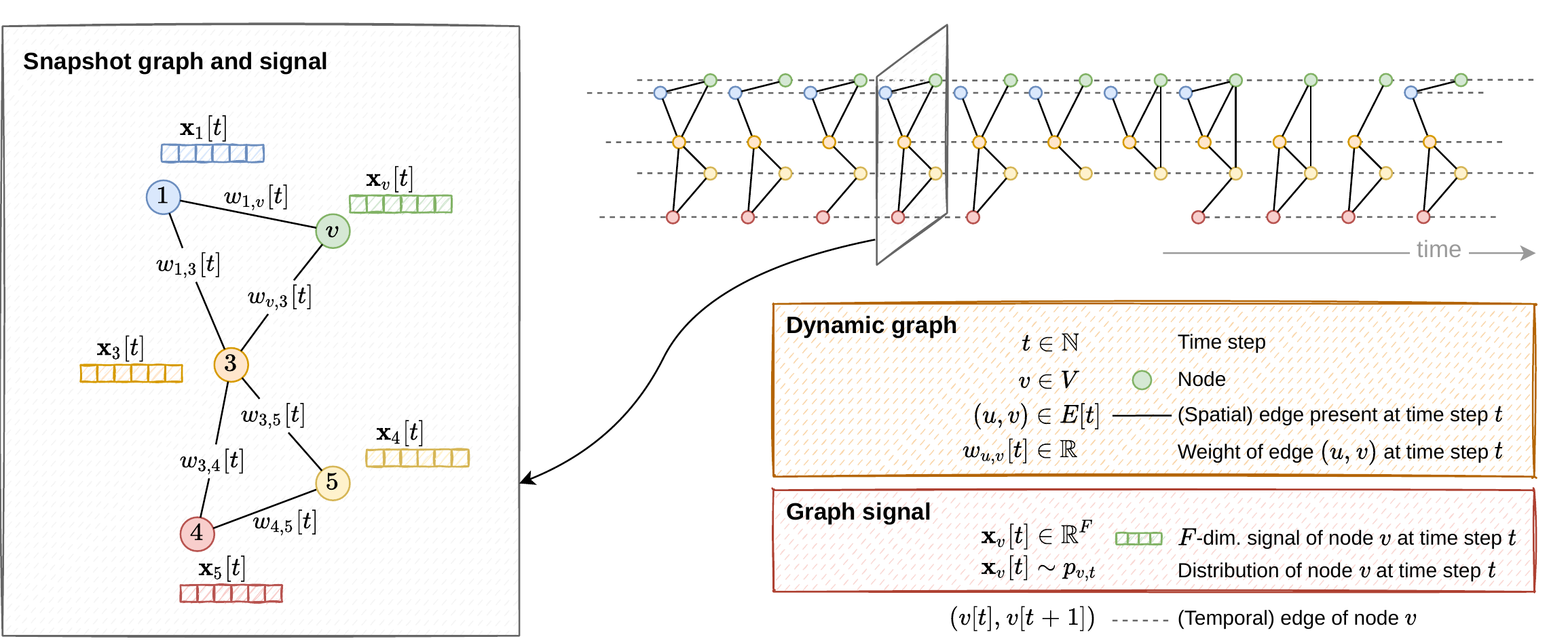}
    \caption{A dynamic weighted graph and a graph signal. The dynamic graph is defined over a set $V$ of nodes; without requesting all nodes to be always available over time. The topology of the graph is represented by solid lines. Dashed lines represent temporal edges connecting the same node at consecutive time steps. The graph signal at generic time step $t$ and node $v$ is multivariate. On the left-hand side, a single snapshot of the temporal graph and graph signal are extracted.}
    \label{fig:graph-signal}
\end{figure}

In the presented setting, we address the problem of assessing the optimality of a predictive model $f_\theta$ trained to solve a forecasting task associated with stochastic graph signal $\X$. We tackle this problem by inspecting the graph signal $\vec R$ composed of the residuals $\vec r_v[t] = \hat \x_v[t] - \x_v[t]$ between the values estimated by model $f_\theta$ and the measured ones in $\X$, respectively.
Indeed, finding data dependency within $\vec R$ entails that there is further structural information that can be learned from the data and, therefore, the predictive model $f_\theta$ can be improved by relying on a better training procedure or considering a more appropriate family of models.  

We propose a novel whiteness test to decide whether a signal on a graph $G$ can be considered white noise or there is prominent evidence of either serial or spatial correlations among data observations. 
In general, testing the assumption of independent signals is ill-posed as it requires studying a number of unknown distributions $p_{v,t}$ (associated with each $\x_v[t]\in\X$) proportional to the number of nodes and time steps, and for each of which we have only a single observation. The test we propose extends traditional methods 
for serial correlation \cite{ljung1978measure,hosking1980multivariate,li2019testing} that study the autocorrelation of the signals at different time lags.
In this paper, we provide methods allowing for testing the whiteness of the graph signal by exploiting known relational inductive biases existing among the data observations represented by the given graph topology and the temporal coherence. In doing so, we expand the analysis of serial correlation to the more general case of spatio-temporal correlation. To the best of our knowledge, this is the first test proposing such a general whiteness test.
The proposed test statistic is based on counting the positive and negative signs of the product between spatially and temporally adjacent observations, whose disproportion indicates either direct or inverse correlation among the variables. Indeed,  authors can consider different statistics.
In summary, the contributions of the paper are the following:
\begin{itemize}
    \item We propose the first statistical test in the literature to check the whiteness hypothesis for a graph signal $\X$
    defined over a, possibly weighted and dynamic, graph $G$ [Sections~\ref{sec:static-test} and \ref{sec:spatio-temporal-test}].
    
    \item We derive the limit distribution of the test statistics under the null hypothesis 
    [Theorem~\ref{theo:test-stat}]. 

    \item We present a statistical procedure based on the proposed test to assess whether a given forecasting model can be considered optimal with respect to the given data or not [Section~\ref{sec:exp:optimality}]. 
    
\end{itemize}

\begin{table}
\caption{List of possible configurations to which the proposed whiteness test is applicable.}
\label{tab:configurations}
\centering
\begin{tabular}{@{}lcccc@{}}
\toprule
\multicolumn{1}{r}{\textbf{Aspect:}} & Temporal dimension    & Edge weights & Node signals   & Neighborhood    \\ \midrule
\multirow{3}{*}{\textbf{Options:}}   & Absent                & Absent       & Scalar ($F=1$) & $1$-hop         \\
                                     & $T>0$, $G$ is static  & Present      & $F>1$          & $K$-hop ($K>1$) \\
                                     & $T>0$, $G$ is dynamic &              &                &                 \\ 
\bottomrule
\end{tabular}
\end{table}

The proposed whiteness test is computationally efficient for sparse graphs as the number of operations scales linearly with the number of edges and time steps. 
Moreover, the test is very general and allows to incorporate very relevant application contexts and operational scenarios as listed in Table~\ref{tab:configurations}.
In particular, the test is applicable when $\x_v[t]$ is univariate or multivariate, when $G$ is static or dynamic, when the edges of $G$ are weighted, and when the graph signal is composed of a set of time series or it is static, namely, $\X=\{\x_v\in\R^F\st v\in V\}$, and there is no temporal dimension involved.

The remainder of the paper is structured as follows.
Section~\ref{sec:related-work} reviews the related work.
Section~\ref{sec:static-test} presents the test in the simplified case of static graph and signal.
Section~\ref{sec:spatio-temporal-test} shows how the test from Section~\ref{sec:static-test} can be applied to spatio-temporal signals on dynamic graphs.
Section~\ref{sec:experiments} reports empirical evidence of the statistical power of the test and shows how to assess the optimality of forecasting models by applying the test to the prediction residuals.
Finally, Section~\ref{sec:conclusions} draws some conclusions and provides pointers to future research.

\section{Related Work}
\label{sec:related-work}

Among the most renowned whiteness tests, there are the Durbin-Watson test \cite{durbin1950testing} and the Ljung-Box test \cite{ljung1978measure}. Both have been introduced to test serial dependency in a univariate time series. In particular, the Ljung-Box test \cite{ljung1978measure} is able to inspect the autocorrelation of the signal at multiple time lags at the same time; these types of tests are known as ``portmanteau'' tests. More recently, \citet{drouiche2000new} proposed a test that, differently from the previous ones, operates on the spectral density of the signal which, under the null hypothesis of white noise, has to be constant.

Several whiteness tests for multivariate data have been developed as extensions of the Box-Pierce or Ljung-Box tests \cite{chitturi1974distribution,hosking1980multivariate,li1981distribution}. These tests operate in the asymptotic regime where the data dimensionality, say $D$, is negligible with respect to the number $T$ of available temporal observations. More recently, the literature has proposed tests overcoming the lack of power and poor approximation of the critical region. Examples are the test by \citet{li2019testing,bose2020whiteness} operating under the condition where the ratio $D/T$ approaches a constant, as $T$ approaches infinity.

To the best of our knowledge, no work has presented a whiteness test for spatio-temporal signals, and what proposed here is pioneering in this direction.

For other purposes, few tests consider a graph structure to represent relationships among the given observations. 
For example, Friedman-Rafsky test \cite{friedman1979multivariate}  compares the distribution of two given sets of data by constructing a minimum spanning tree connecting the given data. The test statistic is based on counting the number of connected components (trees) left after removing the edges linking data points coming from different samples. Friedman-Rafsky test generalizes the run test by \citet{wald1940test}.
Other statistical tests based on topological information computed from the data have been proposed by \citet{rosenbaum2005exact} and \citet{chen2013graph}.

In this paper, we take inspiration from Geary test \cite{geary1970relative} for uncorrelated residuals in univariate data that counts the number of sign changes between consecutive observations. \citet{geary1970relative} shows that, empirically, the test performance was on par with more sophisticated tests, such as the Durbin-Watson test \cite{durbin1950testing}. 
Geary's comparison considered another test, the Wald-Wolfowitz run test \cite{wald1940test} for scalar sequences, which counts the number of runs (subsequences of values all of the same sign) and, hence, closely related to the number sign changes. It turned out that the Walf-Wolfowitz test performed slightly worse than the Geary test, even though the run test does not assume that the observations have an equal probability of being positive and negative.

The statistical test we propose shares the fundamental idea of counting signs of the Geary test, but introduces several advancements that make it a substantial and original contribution. Apart from the fact we are considering graphs, which is a major contribution per se, the most important result is the applicability of the test to spatio-temporal signals defined over graphs. Secondly, the designed test is general enough to operate on weighted and dynamic graphs with multivariate node signals.

\section{Whiteness test for graph signals}
\label{sec:static-test}

\begin{wrapfigure}[23]{R}{.32\textwidth}
    \centering
    \includegraphics[width=.33\textwidth]{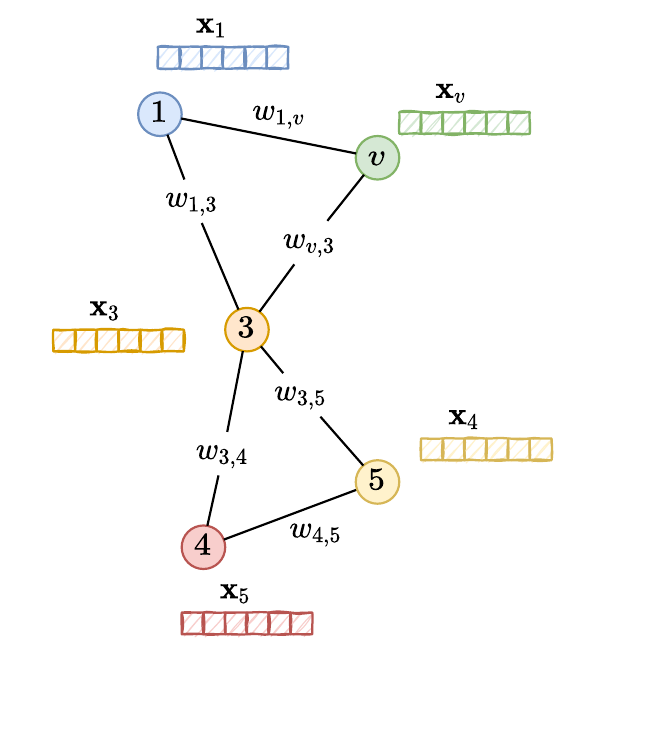}
    \vspace{-1cm}
    \caption{A static graph signal is defined over a weighted graph $G=(V,E,\W)$. Multivariate node signals $\x_v$, $v\in V$, are generic random vectors and, here, not associated with time series.}
    \label{fig:graph-signal-static}
\end{wrapfigure}

For the sake of clarity, we first present the test in a simplified --- yet relevant --- setting characterized by a static weighted graph and a static signal with a single observation (scalar or vector) associated with each node; an example is given in Figure~\ref{fig:graph-signal-static}.  
The most general case where time is involved and the graph is dynamic, as in the scenario depicted in Figure~\ref{fig:graph-signal}, is provided in next Section~\ref{sec:spatio-temporal-test} as an extension of what we are presenting here. In particular, we will see that we can apply the test here designed by considering a suitably constructed (multiplex) graph involving temporal edges alongside the spatial edges present in the original graph.

Consider a weighted graph $G=(V,E,\W)$ defined over node set $V$, edge set $E\subseteq V\times V$ and scalar weights $$\W=\{w_{e}\in\R_+\st e\in E\}.$$ The graph topology and weights define the spatial structure underlying a graph signal $$\X=\{\x_v\in\R^F\st v\in V\}$$ defined over the nodes of $G$.
We assume that the edge weights are positive values and encode the strength or the capacity of the links; without loss of generality, we consider absent those edges characterized by null weight. 
Node signals $\x_v\in\R^F$ can be scalars ($F=1$) or vectors ($F>1$), but they are static, meaning that no temporal information is associated with them, as shown in Figure~\ref{fig:graph-signal-static}. Essentially, we consider a single snapshot of the dynamic graph depicted in Figure~\ref{fig:graph-signal}. 

The ultimate goal is to test if \emph{graph} signal $\X$ is white noise, implying that there is no further dependency left among node signals in $\X$. 
The statistical hypotheses of the test are 
\begin{equation}
\label{eq:test-H0-H1}
\begin{cases}
H_0:& 
\x_u,\x_v \text{ are independent for all } u\ne v\in V;
\\
H_1:& 
\x_u,\x_v \text{ are dependent for some } u\ne v\in V.
\end{cases}
\end{equation}
We propose the following whiteness test 
\begin{equation}
\label{eq:test-static}
\text{If } |C_G(\X)| > \gamma \implies \text{ Reject null hypothesis } H_0,
\end{equation}
based on statistic
\begin{align}
\label{eq:test-stat}
C_G(\X) &= \frac{\widetilde C_G(\X)}{W_2^{{1}/{2}}}.
\end{align}
Statistic $C_G(\X)$ is asymptotically distributed a standard Gaussian $\mc N(0,1)$ as $|E|\to \infty$ (we prove it in following Theorem~\ref{theo:test-stat}), 
and  is defined over quantities
\begin{align}
\label{eq:Ctilde}
\widetilde C_G(\X) &= \sum_{(u,v) \in E} w_{u,v}\;\sign(\x_u^\top \x_v),
\\
\label{eq:W2}
W_2 &= \sum_{{(u,v)\in E}\;:\,{u<v}} (w_{u,v} + w_{v,u})^2,
\end{align}
with $w_{u,v}=0$ if $(u,v) \not \in E$,   
and sign function 
$\sign(x) = 1$ if $x>0$, $\sign(x) = -1$ if $x<0$, and $\sign(x) = 0$ otherwise;
note that for undirected graphs $W_2$ reduces to $\sum_{(u,v)\in E} w_{u,v}^2$.

The intuition behind \eqref{eq:test-static} to test null hypothesis $H_0$ against $H_1$ in \eqref{eq:test-H0-H1} is that, when the node signals in $\X$ are mutually independent and centered around zero, then random variables $\sign(\x_u^\top \x_v)$ for all $u\ne v\in V$ should also be centered on zero, as we prove in Lemma~\ref{lemma:sign-dist}, Appendix~\ref{sec:proofs}. 
In particular, in the case of scalar signals, \ie, $F=1$, we can factorize $\sign(\x_u\,\x_v)=\sign(\x_u)\,\sign(\x_v)$ so that it
represents a sign change between the signals of nodes $u$ and $v$.
We observe that large values of $C_G(\X)\gg 0$ indicate the presence of few sign changes between observations of adjacent nodes, which in turn suggests correlation among variables. Similarly, inverse correlation is revealed by $C_G(\X)\ll 0$. Therefore, $|C_G(\X)| \gg 0$ is symptomatic of  dependent data, and justifies our test in \eqref{eq:test-static}. 
When $F>1$, we can interpret the scalar product  $\x_u^\top \x_v$ as a correlation between the node signals so that, under the null hypothesis, the expected value of the signs is null and the observed values of $C_G(\X)$ are around zero. 
Another interpretation is more geometrical and treats $\x_u, \x_v$ as points in a Euclidean space. Here, the sign of the scalar (inner) products $\x_u^\top \x_v$ indicates whether the two vectors point in similar or opposing directions, concluding once again that $|C_G(\X)|$ is expected to be small when the node signals are independent.

Next Theorem~\ref{theo:test-stat} supports the soundness of test \eqref{eq:test-static} and provides a distribution-free criterion to select threshold $\gamma$ granting a user-defined significance level $\alpha=\prob(\text{Reject }H_0| H_0)$, \eg, $\alpha=0.05$. 
\begin{theorem} 
\label{theo:test-stat}
Consider a weighted graph $G=(V,E,\W)$ without self-loops and a stochastic graph signal $\X=\{\x_v\in \R^F\st v\in V\}$ on it, with $\vec x_v\ne 0$ almost surely.
Under assumptions
\begin{enumerate}[label=\textnormal{({\bfseries Ass\arabic*})}, leftmargin=2cm]
\item \label{hp:null} 
$\{\x_v \st v\in V\}$ are all mutually independent (hypothesis $H_0$ in \eqref{eq:test-H0-H1}),
\item \label{hp:median} 
$\EE_{\x_v}\left[\sign(\bar \x^\top \x_v)\right]=0$ for all $\bar \x\in\R^F\setminus\{\vec 0\}$ and $v\in V$,
\item \label{hp:weights} 
$w_{u,v} \in (0, w_+]$ for all $(u,v)\in E$, and $W_2\to \infty$ as $|E|\to \infty$,
\end{enumerate}
the distribution of $C_G(\X)$ in \eqref{eq:test-stat} converges weakly to a standard Gaussian distribution $\mc N(0,1)$ as the number $|E|$ of edges goes to infinity.
\end{theorem}

In light of above Theorem~\ref{theo:test-stat}, threshold $\gamma$ is selected to be the quantile $1-\alpha/2$ of the standard Gaussian distribution so that
$$
\prob(|C_G(\X)| > \gamma \st H_0) = \alpha,
$$
thus ensuring to meet the user-defined significance level $\alpha$.
Before sketching the proof, we comment on the assumptions under which the theorem is derived.

\begin{remark}[Assumption of no self-loops and {$\prob(\x_v=0)=0$}]
\label{remark:hp:null}
Quantities $\sign(\x_v^\top\x_v)$ arising from self-loops do not carry relevant information to test the independence of node signals. Moreover, we could have expressed \eqref{eq:Ctilde} as a sum over all edges that are not self-loop. Therefore, the assumption of no self-loop is requested to simplify only the notation. 
Secondly, requesting null probability of null node signals $\x_v$ is another simplifying assumption that eliminates the probability that $\sign(\x_v^\top\bar \x)=0$. 
\end{remark}

\begin{remark}[Assumption~\ref{hp:median}]
\label{remark:hp:median}
Assumption~\ref{hp:median} requests that all node signals have a distribution yielding the same probability of being in either of the two half-spaces of $\R^F$ defined by the sign of the scalar product with $\bar \x$. 
Note that in the scalar case with $F=1$, \ref{hp:median} implies that $\prob(\x_v>0)=\prob(\x_v<0)$, which is equivalent to asking that the median of all $\x_v$ is zero.
Therefore it does not impose any real constraint on the data distribution because: (i) if the median is not zero, then the model is already not optimal in the sense of the mean absolute deviation and (ii) we can apply the test to the graph signal translated to have zero median.
Conversely, for $F>1$, it is not always possible to center the data so that \ref{hp:median} holds.%
    \footnote{A counterexample is given by considering all components of all node signals to be \iid with distribution $1/2\, U[-4,0) + 1/2\, U[0 ,1)$, \ie a mixture of two uniform distributions. For $F=2$, we have that if \ref{hp:median} holds for $\bar\x=[0,1]^\top$ and $[1,0]^\top$, then it cannot hold also for $\bar\x=[1,1]^\top$.}
Nevertheless, in this case we can perform separate tests, one for each of the $F$ components, and then either study them individually, employ some multiple hypothesis test correction (\eg, \cite{hochberg1988sharper}), or sum them together if the components are independent.
In any case, this assumption allows node signals to have different distributions. 
\end{remark}

\begin{remark}[Assumption~\ref{hp:weights}]
\label{remark:hp:weights}
Finally, in Assumption~\ref{hp:weights} we request the weights to be positive. We stress that, as already mentioned, the edge weights are assumed to encode the strength of the relation between the corresponding node so that higher weights imply a stronger impact on statistics \eqref{eq:test-stat}. An example is when weights come from the absolute value of Pearson's correlation between signals. 
Not rarely, however, the graph we are given does not come with edge weights or the provided edge attributes do not reflect the criterion assumed above. In all such situations, we can still apply test~\eqref{eq:test-static} straightforwardly considering that all weights are equal to $1$; accordingly, constant $W_2$ in \eqref{eq:W2} becomes equal to the number $|E|$ of edges in the graph. Other criteria to re-weight the graph can be designed for specific cases.
The assumption of bounded weights, instead, is technical and only takes part in the limit case of $|E|\to \infty$. Intuitively, \ref{hp:weights} ensures that all edges bring a tangible contribution to the final statistics $C_G(\X)$. The same theorem can be proven under milder assumptions; see Appendix~\ref{sec:proofs}. 
\end{remark}

\begin{proof}[Sketch of the proof.]
The proof of Theorem~\ref{theo:test-stat} is based on the fact that, under \ref{hp:null} and \ref{hp:median}, random variables $\sign(\x_v^\top \x_u)$, for all $(u,v)\in E$, are mutually independent, including edges that share one of the two ending nodes. 
It follows that statistic $\widetilde C_G(\X)$ in \eqref{eq:Ctilde} is a weighted sum of independent Bernoulli random variables. Finally, we prove that, with bounded weights, the Lindeberg condition [Equation~\ref{eq:lindeberg-condition}, Appendix~\ref{sec:proofs}] holds for $\widetilde C_G(\X)$ under \ref{hp:weights}, and we conclude the thesis by applying the central limit theorem \cite{billingsley1995probability} for independent, but not equally distributed, summands of $\widetilde C_G(\X)$.
Detailed development of the proof is given in Appendix~\ref{sec:proofs}.
\end{proof}

Despite the apparently little amount of information required to perform the test (only the sign of the scalar products $\x_v^\top\x_u$ of all edges $(u,v)\in E$), \citet{geary1970relative} have shown that for scalar time-series (equivalent to a line graph) the test performed on par of more sophisticated tests. We give empirical evidence of the effectiveness of our whiteness test in Section~\ref{sec:experiments}.

\subsection{Sparse versus fully connected graphs}
\label{sec:sparse-vs-full}

\begin{figure}
\centering
\includegraphics[width=.23\textwidth]{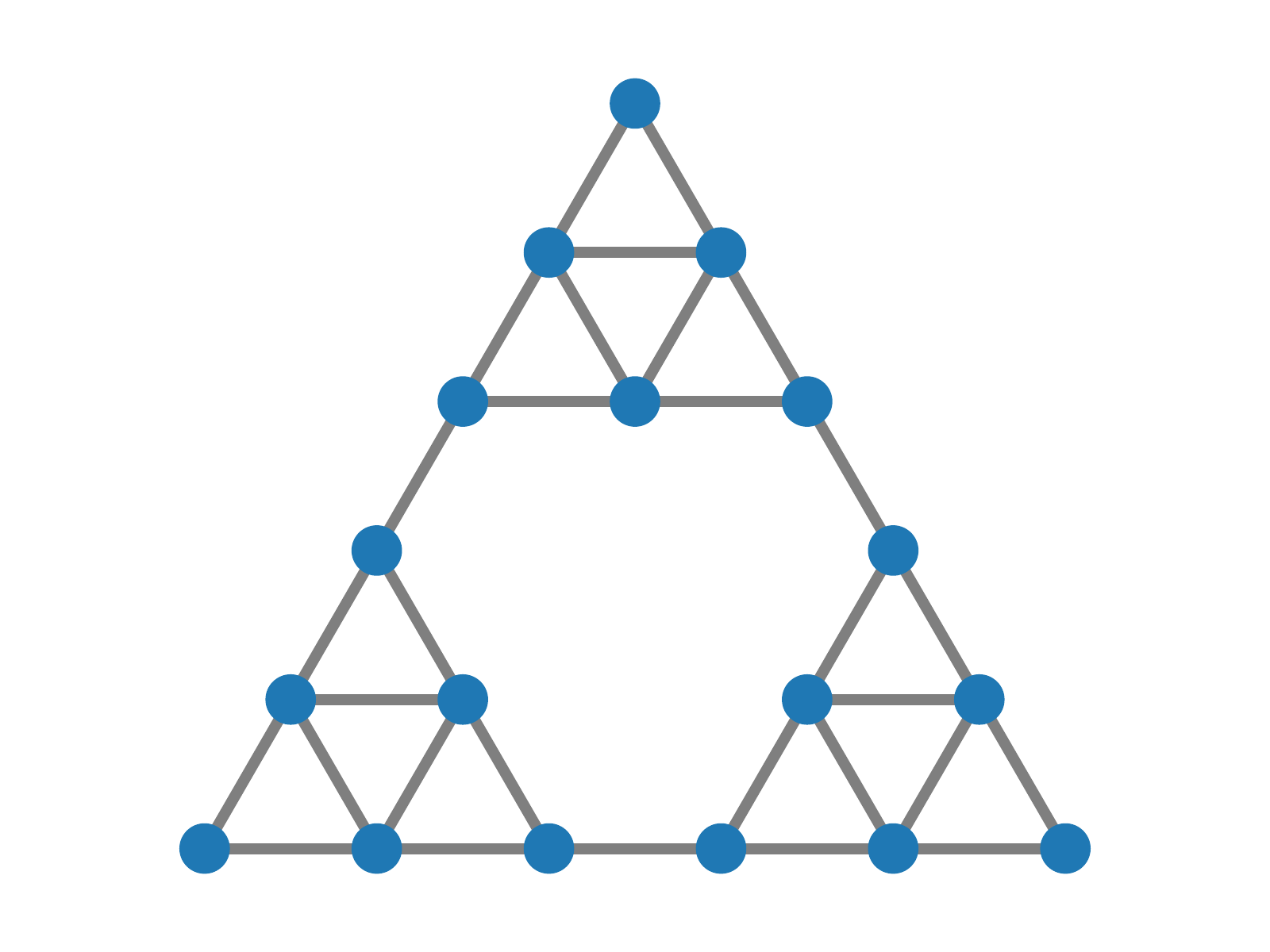}
\includegraphics[width=.56\textwidth]{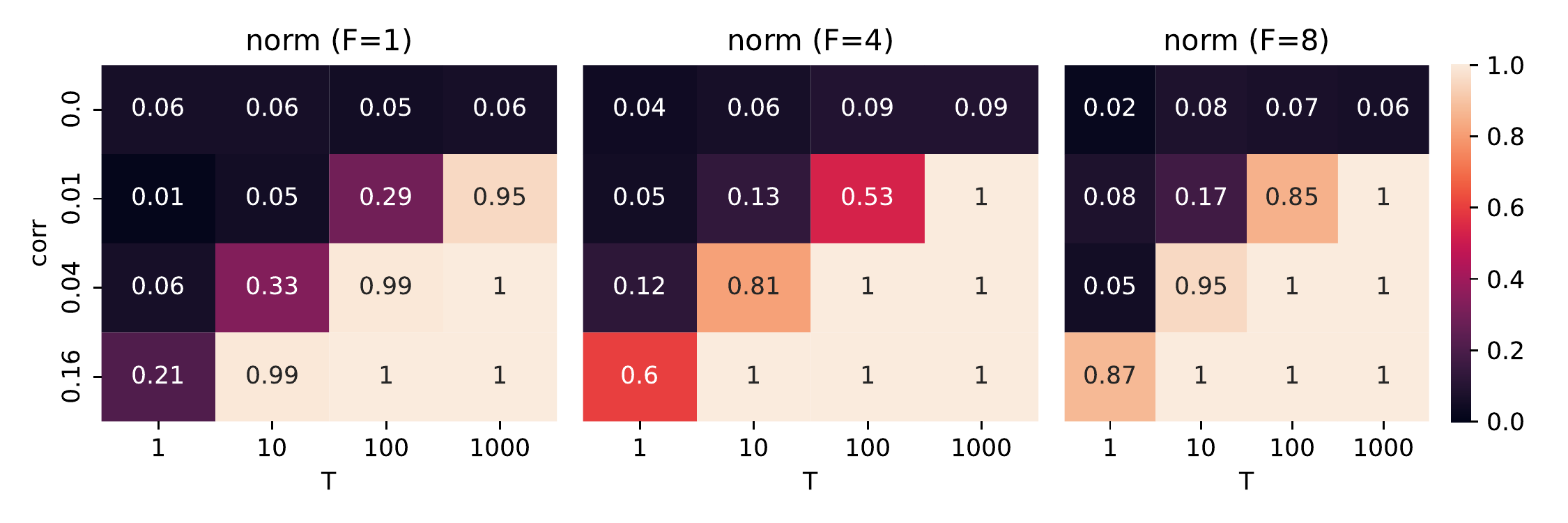}
\\
\includegraphics[width=.23\textwidth]{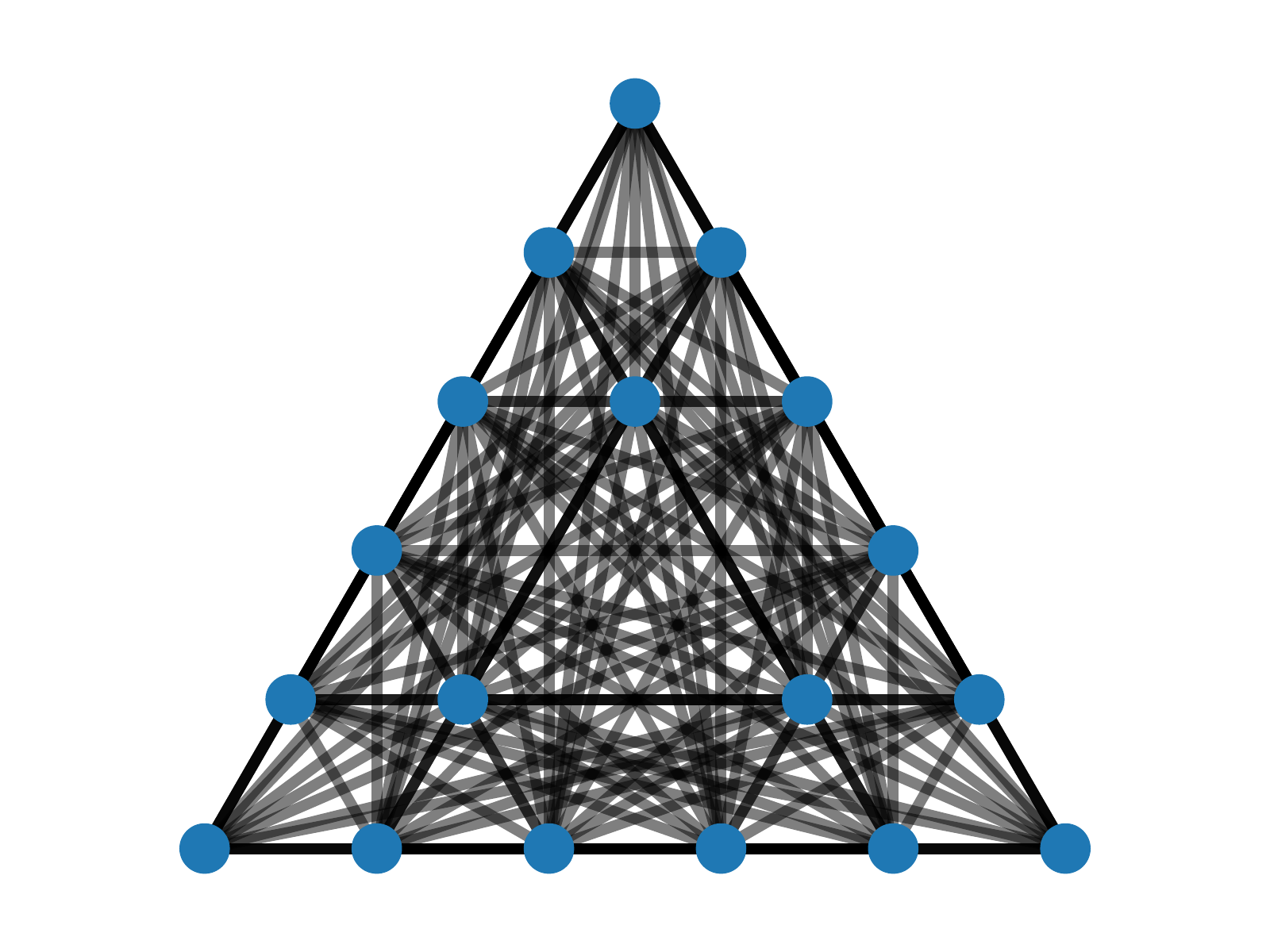}
\includegraphics[width=.56\textwidth]{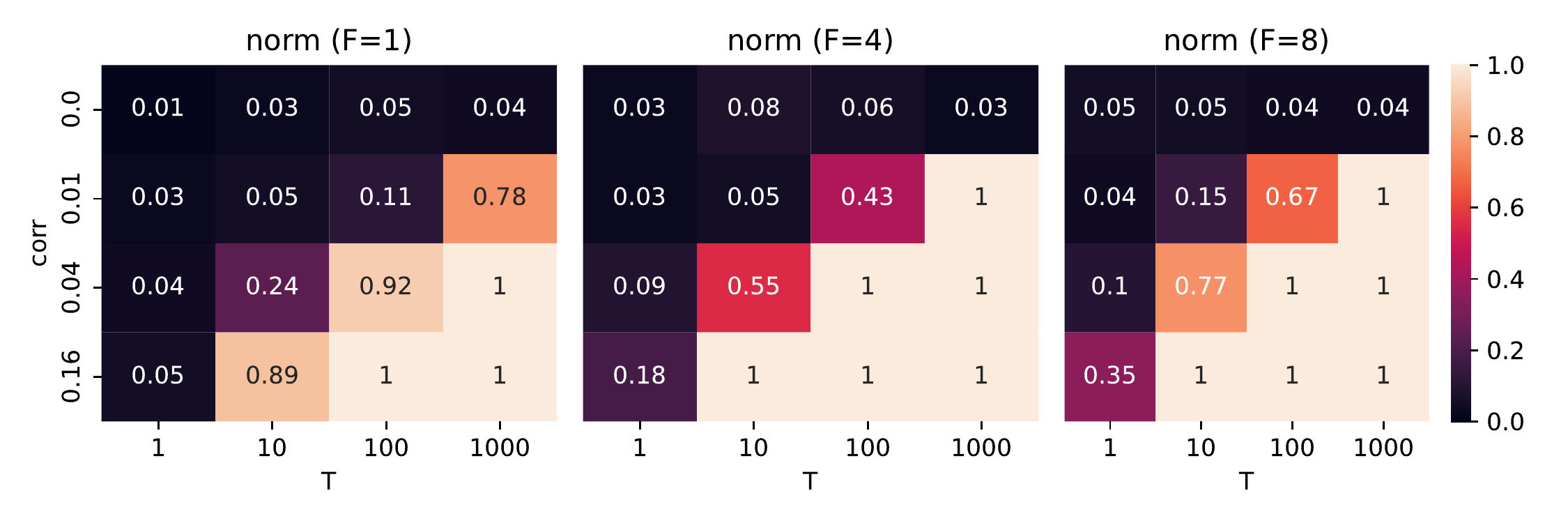}
\caption{Rate of rejected null hypotheses for different node signal dimensions $F$, number of time steps $T$, and correlation parameters $c_\spa=c$ and $c_\tim=0$; a detailed description of the data generating process is provided in Section~\ref{sec:exp:correlation}. The first row corresponds to the summation in \eqref{eq:Ctilde} only over the edges of the original graph, while the second one to a summation over all pairs of nodes (fully connected graph).}
\label{fig:sparse-vs-complete}
\end{figure}

We observe that, if we were to consider all pairs of nodes $(u,v)\in V\times V$ in \eqref{eq:Ctilde}, instead of only the edges present in graph $G$, then several additional terms that, usually, carry negligible functional dependency would be considered in the sum \eqref{eq:Ctilde}.
As a result, the impacts of more sensible correlations between nodes close in the graph can wash out.
We empirically demonstrate this effect in Figure~\ref{fig:sparse-vs-complete} where, for the same configurations of the simulated graph signal, considering only the edges in $E$ instead of all possible pairs of nodes, results in a higher statistical power, namely, a higher rate of correctly rejecting the null hypothesis.

Moreover, from formula \eqref{eq:Ctilde}, we immediately see that the computational complexity of the test scales linearly in the number of edges $|E|$. Therefore, the test is more scalable for  sparse graphs.

\subsection{Whiteness test under different problem settings}

Referring to Table~\ref{tab:configurations}, we see that test \eqref{eq:test-static} has been presented for 
\begin{enumerate*}[label=(\roman*)]
\item static graph signals (absent temporal dimension),
\item weighted and unweighted graphs,
\item scalar and multivariate node signals, and
\item $1$-hop neighbors.
\end{enumerate*}
In this section, we conclude all scenarios where the temporal dimension is absent and leave the case of temporal signals and dynamic graphs to Section~\ref{sec:spatio-temporal-test}. 

Typically, the correlation between node signals decays as the get farther apart in graph.
Therefore, the proposed test considers dependencies only among adjacent nodes, that is, $1$-hop neighboring nodes, which should be those displaying the strongest dependency. However, 
when we have reason to think that the dependency among signals covers more than $1$-hop neighbors, we can consider $K$-hop relationships as well for $K>1$. To do so, it is enough to extend the edge set $E$ with $\cup_{k=1}^K E_k$ where
$$E_k=\left\{(v_1,v_k)\st \exists (v_1,\dots,v_k) \text{ with distinct } v_1\dots v_k\in V \text{ and } (v_i,v_{i+1})\in E, \forall i<k\right\},$$
that is, all pairs of nodes connected by paths of length $k\le K$. 
The presented strategy can be seen as an extension to graphs of the portmanteau lack-of-fit test presented in \cite{box2015time}.
When appropriate, edges can be weighted differently, for instance, according to the associated path length.

From the current section and Remark~\ref{remark:hp:weights}, we see that, before applying test~\eqref{eq:test-static}, graph $G$  can be suitably modified, or augmented, to best fit the application at hand. In the next section, we demonstrate how to apply test~\eqref{eq:test-static} to temporal signals and dynamic graphs, hence covering the remaining configurations in Table~\ref{tab:configurations}.

\section{Whiteness test for spatio-temporal graph signals}
\label{sec:spatio-temporal-test}

In this section, we show how to apply the test from previous Section~\ref{sec:static-test} to a generic spatio-temporal signal associated with a possibly dynamic and weighted graph.

We consider a time frame $[1, T]$, and a discrete-time dynamic graph 
\begin{equation}
\label{eq:dynamic-graph}
G=\left(V[t],E[t],\W[t]\st t=1,\dots,T\right).
\end{equation}
At each time step $t$, $(V[t], E[t],\W[t])$ is a (static) weighted graph like the one depicted in Figure~\ref{fig:graph-signal-static}. Since we are dealing with node-level time series, we have a correspondence between the nodes of $G$ at different time steps and, in general, $V[t_1]\cap V[t_2]\ne \emptyset$ if $t_1\ne t_2$. Note that the possibility of having a variable node set is important to model scenarios like the integration of new intersections in a street map, the removal of a broken device in a cyber-physical system, or missing data from smart meters in a power grid.
The edge weights can be dynamic, as well, and we denote with $w_{u,v}[t]\in\W[t]$ the weight associated with edge $(u,v)$ present in the edge set $E[t]$ at time $t$.   

On top of graph $G$, we consider a stochastic graph signal 
\begin{equation}
\label{eq:temporal-graph-signal}
\X = \left\{\x_v[t]\in\R^F\st v\in V[t],\; t=1,2,\dots,T\right\}.
\end{equation}
In this new setting, each node $v$ of the graph is associated with a (multivariate) time series $\x_v[t]$ available for all time steps $t=1,2,\dots, T$, if $v\in V[t]$ for all $t$, or only for a subset of $\{1,2,\dots,T\}.$ 
This setting is depicted in Figure~\ref{fig:graph-signal} and expresses the full potential of the whiteness test we propose.

Here, we are interested in identifying not only spatial dependencies among different nodes, but also temporal ones appearing among observations from the same node at different time steps.
Accordingly,  we adapt the null and alternative hypotheses in \eqref{eq:test-H0-H1} to account for the temporal dimension, and define
\begin{equation}
\label{eq:test-H0-H1-temporal}
\begin{cases}
H_0:& 
\text{All pairs }\x_u[t_u],\x_v[t_v]\in\X, \text{ with } (u,t_u)\ne (v,t_v), \text{ are independent;}
\\
H_1:& 
\text{Some pairs }\x_u[t_u],\x_v[t_v]\in\X, \text{ with } (u,t_u)\ne (v,t_v), \text{ are dependent.}
\end{cases}
\end{equation}
By suitably constructing a static graph $G_T=(V_T,E_T,\W_T)$ from $G$ as a stack of snapshot graphs $(V[t],E[t],\W[t])$ and a graph signal $\X_T$ with the same observations of $\X$, but located at the nodes of graph $G_T$, we are able to test hypotheses $H_0,H_1$ with the static test \eqref{eq:test-static} presented in previous Section~\ref{sec:static-test}. 
The proposed whiteness test for spatio-temporal graph signals becomes 
\begin{equation}
\label{eq:test-temporal}
\text{If } |C_{G_T}(\X_T)| > \gamma \implies \text{ Reject null hypothesis } H_0,
\end{equation}
where $C_{G_T}(\X_T)$ is statistic \eqref{eq:test-stat} applied to $G_T$ and $\X_T$, and threshold $\gamma$ is derived from Theorem~\ref{theo:test-stat} according to a user-defined significance level $\alpha$. 

As we detail in next Section~\ref{sec:multiplex}, the construction of $G_T$ and $\X_T$ 
results in a single larger graph preserving all spatial and temporal relations in $G$ and $\X$, for which there is no constraint on the node and edge sets to be constant throughout the entire time frame $[1, T]$.
Whiteness test \eqref{eq:test-temporal} completes the coverage of the configurations reported in Table~\ref{tab:configurations} and, to emphasize its generality and versatility, we name it AZ-test.

\subsection{Multiplex graph and associated signal}
\label{sec:multiplex}

The idea is to construct graph $G_T=(V_T,E_T,\W_T)$ as a multiplex graph by stacking graphs $(V[t], E[t], \W[t])$ of dynamic graph  \eqref{eq:dynamic-graph} according to their temporal index $t$.
A representation of graph $G_T$ is given in Figure~\ref{fig:graph-signal}, where
the nodes of $G$ are replicated for all their occurrences across time preserving their spatial connectivity (solid lines), and temporal edges (dashed lines) are added by connecting corresponding nodes at subsequent time steps.
When graph $G=(V,E,\W)$ is static, we simply replicate it for all considered time steps.
Formally, $G_T$ and $\X_T$ are constructed as follows.

\begin{description}
\item[Nodes] 
We define node set $V_T$ as the disjoint union of $V[t]$, for all $t$, that is,
$$
V_T = \left\{v[t] := (v,t) \st v\in V[t], t=1,2,\dots,T\right\}.
$$ 
If node $v$ is not present in node set $V[t']$, for some $t'$, then $(v,t')\not \in V_T$. Therefore, the number of nodes $|V_T|$ is equal to $\sum_t |V[t]|$ and, for a static graph $G$, we would have $|V_T|=|V|\cdot T$.

\item[Edges] 
The edge set $E_T$ is composed of two types of edges: spatial and temporal edges. The set of spatial edges is the disjoint union of sets $E[t]$, for all $t$, that is
$$
E_\spa = \left\{(u[t], v[t]) \st (u,v) \in E[t],\; t=1,2,\dots, T\right\}.
$$
Temporal edges, instead, connect consecutive occurrences of the same nodes:
$$
E_\tim = \left\{(v[t], v[t+1]) \st v\in V[t]\cap V[t+1],\; t=1,2,\dots, T-1\right\}.
$$
The edge set $E_T$ is then given by
$$
E_T = E_\tim \cup E_\spa
$$
Referring to Figure~\ref{fig:graph-signal}, set $E_\spa$ collects all edges represented by solid lines, while edges in $E_\tim$ are represented as dashed lines. 

\item[Weights] 
The edge weights in $\W_T$ are constructed similarly to $E_T$: edges $(u[t], v[t]) \in E_\spa$ are associated with weight $w_{u,v}[t]\in \W[t]$ corresponding to edge $(u,v)\in E[t]$, while the weight of the temporal edges can be set arbitrarily. In Section~\ref{sec:weights-temporal}, we propose a reasonable way to define a single weight $w_\tim$ for all temporal edges in $E_\tim$.

\item[Signals]
Defined $G_T$, the graph signal $\X$ in \eqref{eq:temporal-graph-signal} can be arranged over the new node set $V_T$: 
$$ 
\X_T=\left\{\x_{v[t]}:=\x_v[t] \in \R^F \st v[t]\in V_T,\text{ with } \x_v[t]\in\X\right\}.
$$
\end{description}

To conclude, we observe that when $T=1$ the described procedure degenerates to that presented in Section~\ref{sec:static-test}. Secondly, we note that the construction of $G_T$ and $\X_T$ is convenient for the theoretical development because it allows applying all the results from Section~\ref{sec:static-test}, however, a practical implementation does not need the full $G_T$ graph to be explicitly constructed, as will be more evident from the next section.

\subsection{Weights for temporal edges}
\label{sec:weights-temporal}

Commenting further about weighting the temporal edges, we note that by definition of $G_T$ 
there is typically a larger number of spatial edges, than temporal edges; in the case of a static unweighted graph $G$, $|E_\tim|=(T-1) \cdot |V|$, while $$|E_\spa|=T \cdot |V| \cdot d \approx d \cdot |E_\tim|,$$ where $d$ is the average degree of the nodes in $G$.
Consequently, an imbalance of positive and negative signs encountered along the temporal edges has, overall, a lower impact on the final test statistic $C_{G_T}(\X_T)$ than that of the spatial edges.
Accordingly, we may find it appropriate to weight temporal edges so that spatial and temporal relations are of comparable importance in $C_{G_T}(\X_T)$.
In the remainder of this section, we derive a weight $w_\tim$ that meets the above criterion, when associated with all temporal edges.

Statistics $C_{G_T}(\X_T)$ can be decomposed into a sum of contributions of the spatial and temporal edges, and can be rewritten as
$$
C_{G_T}(\X_T) = 
\frac{
    \widetilde C_\spa + \widetilde C_\tim
}{
    \left(W_{2,\spa} + W_{2,\tim}\right)^\frac{1}{2}
}
$$
where
\begin{align*}
\widetilde C_\spa &=\sum_{(u[t],v[t])\in E_\spa} w_{u,v}[t]\cdot \sign\left(\x_u[t]^\top\x_v[t]\right),
\\
\widetilde C_\tim &=\sum_{(v[t],v[t+1])\in E_\tim} w_\tim\cdot \sign\left(\x_v[t]^\top\x_v[t+1]\right),
\\
W_{2,\spa} &= \sum_{(u[t],v[t])\in E_\spa\,:\, u<v} (w_{u,v}[t] + w_{v,u}[t])^2,
\\
W_{2,\tim} &= |E_\tim|\cdot w_\tim^2;
\end{align*}
as in the definition of $W_2$ in \eqref{eq:W2}, if $(u,v)\not\in E[t]$, then we assume $w_{u,v}[t]=0$.

Under the assumptions of Theorem~\ref{theo:test-stat}, we have that $\EE[\widetilde C_\spa] = 0 = \EE[\widetilde C_\tim],$ while $\var[\widetilde C_\spa] = W_{2,\spa}$ and $\var[\widetilde C_\tim] = W_{2,\tim}.$
We conclude that balancing the spatial and temporal contributions amounts to select $w_\tim$ such that the variance is the same, \ie, $W_{2,\spa}=W_{2,\tim}$. We result in 
$$
w_\tim =  \left(\frac{W_{2,\spa}}{|E_\tim|}\right)^\frac{1}{2}.
$$

When appropriate, a convex combination of spatial and temporal components can be defined in terms of a scalar value $\lambda\in[0,1]$ and obtain statistics
\begin{equation}
\label{eq:test-stat-lambda}
C_{G_T}(\X_T;\lambda) := 
\frac{
    \lambda \;\widetilde C_\spa + (1-\lambda )\;\widetilde C_\tim
}{
    \left(\lambda^2\,W_{2,\spa} + (1-\lambda)^2\,W_{2,\tim}\right)^\frac{1}{2}
}.
\end{equation}
Note that the new statistic defined in \eqref{eq:test-stat-lambda} enjoys all the properties presented in the previous section; in fact, it is equivalent to scaling all the weights of $G_T$ of either $\lambda$ or $(1-\lambda)$, and $C_{G_T}(\X_T;1/2)=C_{G_T}(\X_T)$. In particular, as a corollary of Theorem~\ref{theo:test-stat}, $C_{G_T}(\X_T;\lambda)$ is approximately distributed as a standard Gaussian under the null hypothesis, and it can be directly employed in AZ-test~\ref{eq:test-temporal}.

\section{Experiments}
\label{sec:experiments}

We consider two experimental setups. 
In the first one, Section~\ref{sec:exp:correlation}, we generate correlated graph signals and we study the ability of the AZ-test to detect data dependency of various amplitude.
In the second one, Section~\ref{sec:exp:optimality}, we test the optimality of machine-learning models trained to solve forecasting problems on synthetic and real-world data.

\subsection{Detection of correlated residuals}
\label{sec:exp:correlation}

We consider the undirected unweighted graph $G=(V,E)$ depicted in Figure~\ref{fig:corr-signals} and denote with $\vec A$ its adjacency matrix, that is $\vec A_{u,v}=1$ if $(u,v)\in E$ and $\vec A_{u,v}=0$ otherwise. In this set of experiments, we generate white-noise graph signals and correlated graph signals with the following procedure.

\subsubsection{Signal generation and testing procedure}
\label{sec:exp:correlation:data-gen}

We consider a scalar probability distribution $P$ with null median, and we sample $|V|\cdot T$ independent values $\eta_v[t]$ from $P$, one for all $t=1\dots T$ and $v\in V$. 
We obtain that $\X=\{\x_v[t]=\eta_v[t]\st \forall t,v\}$ is a graph signal of independent observation, meeting our null hypothesis $H_0$ in \eqref{eq:test-H0-H1-temporal}. 

We generate correlated signals by propagating an independent signal $\{\eta_v[t]\st \forall t,v\}$ across graph $G$ and the time dimension. Specifically, we consider the following model
\begin{equation}
\label{eq:corr-signal-generation}
\x_v[t] = \eta_v[t] + c_\tim \;\eta_v[t-1] + c_\spa \left(\sum_{(u,v)\in E} w_{u,v}[t] \,\eta_u[t]\right) - m,
\end{equation}
for all $t=2,3,\dots,T$ and $v\in V$.
The correlation between observations is controlled by positive parameters $c_\spa$ and $c_\tim$.
Finally, a scalar offset $m$ is subtracted to have a process $\X$ of null median, thus removing side effects that could make the AZ-test fail.

We consider different distributions $P$, including unimodal symmetric distributions, as well as asymmetric and bimodal ones. In particular, we select the standard Gaussian distribution $\mc N(0,1)$, the chi-squared distributions%
    \footnote{Distributions $\chi_2(d)$ are shifted to have null median.}
$\chi_2(d)$ with $d=1,5$ degrees of freedom, a mixture of $\mc N(-3,1)$ and $\mc N(+3,1)$, a mixture of $\chi_2(1)$ and $-\chi_2(5)$, and a mixture of uniform distributions $U[-4,0)$ and $U[0,1)$. 
We generate graph signals for variable number $T$ of time steps, dimension $F$ of the node features, and parameters $c_\spa$ and $c_\tim$. We also include the case with $c_\spa=c_\tim=0$ to test the rate of false positives, and to verify that we are able to replicate the user-defined significance level $\alpha$. 

Figure~\ref{fig:corr-signals} displays sample signals generated from different values of $c$. 
We observe that the higher the value of $c$, the more correlated the signals appear on both the temporal and spatial axes.
Temporal correlation is visible as horizontal strikes of similar color in the heatmaps. 
Instead, having numbered the nodes so that neighboring nodes have similar indices,%
    \footnote{We numbered the nodes according to the magnitude of components of the eigenvector associated with the smallest non-null eigenvalue of the Laplacian matrix $\vec L$ of $G$. 
    The Laplacian of an undirected graph with adjacency matrix $\vec A$ is defined by $\vec L=\vec D-\vec A$, where $\vec D_{ij}=0$ if $i\ne j$, and $D_{ii}=\sum_{j}\vec A_{ij}$.}
the spatial correlation is suggested by vertical strikes.

\begin{figure}
\centering
\rotatebox{90}{\parbox{3cm}{\centering$c=0.0$}}
\includegraphics[height=3.2cm]{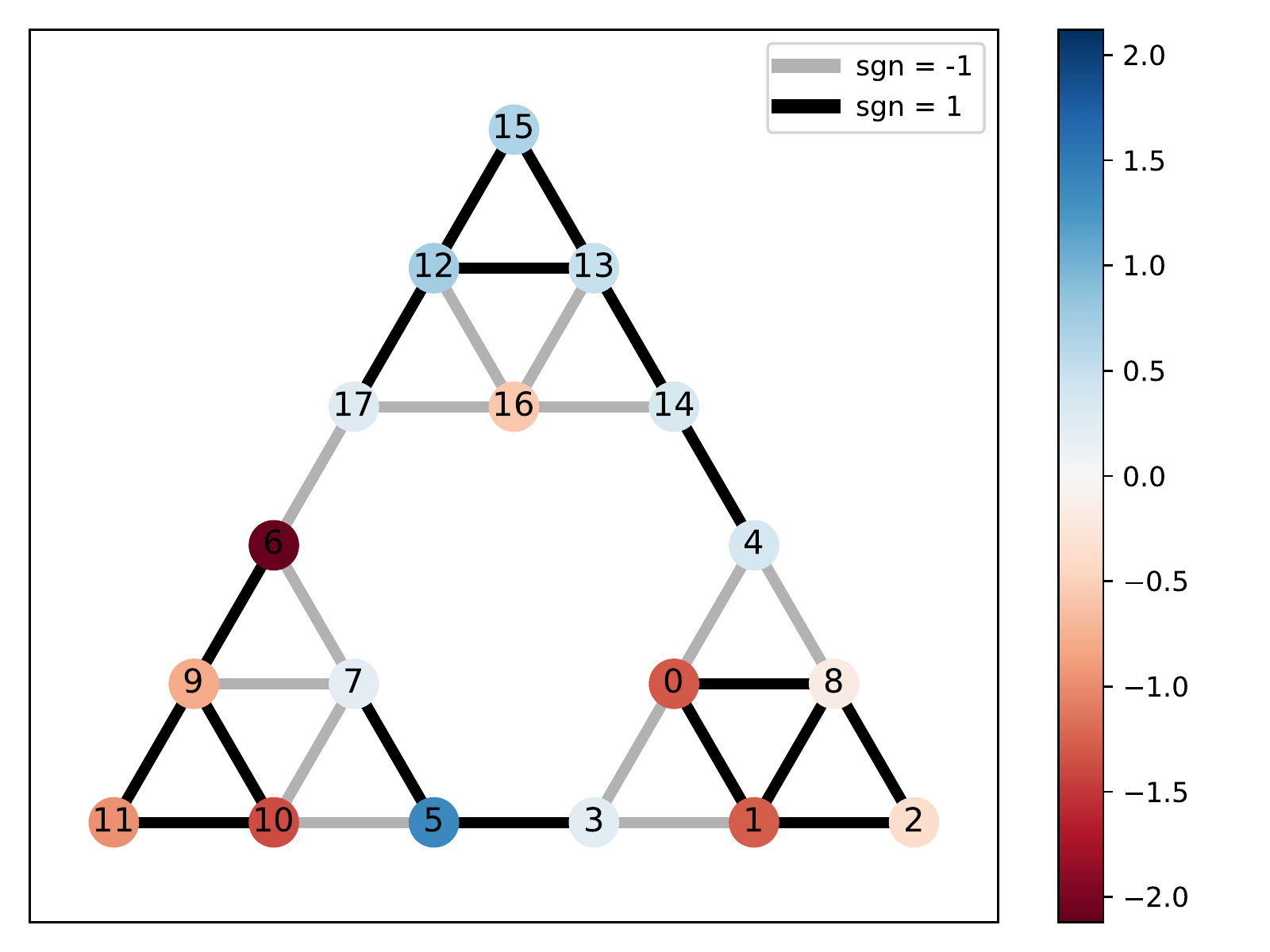}
\includegraphics[height=3.2cm]{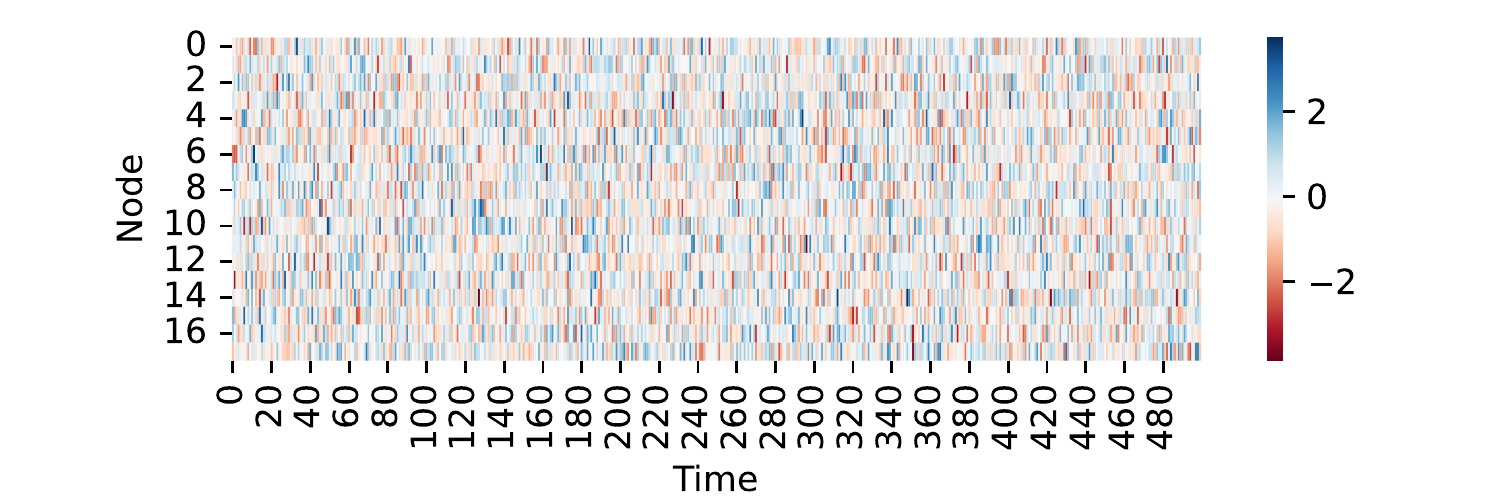}\\
\rotatebox{90}{\parbox{3cm}{\centering$c=0.04$}}
\includegraphics[height=3.2cm]{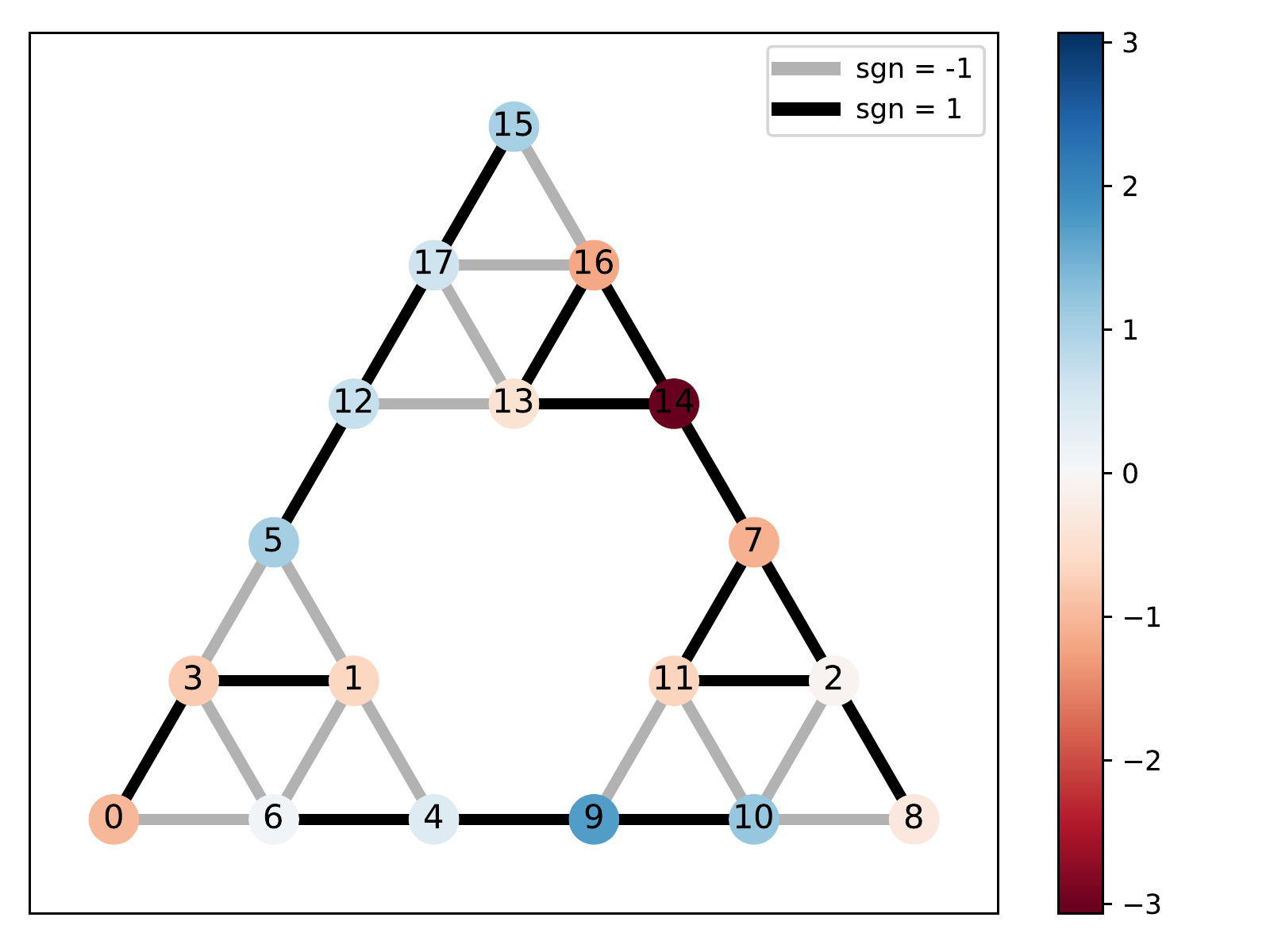}
\includegraphics[height=3.2cm]{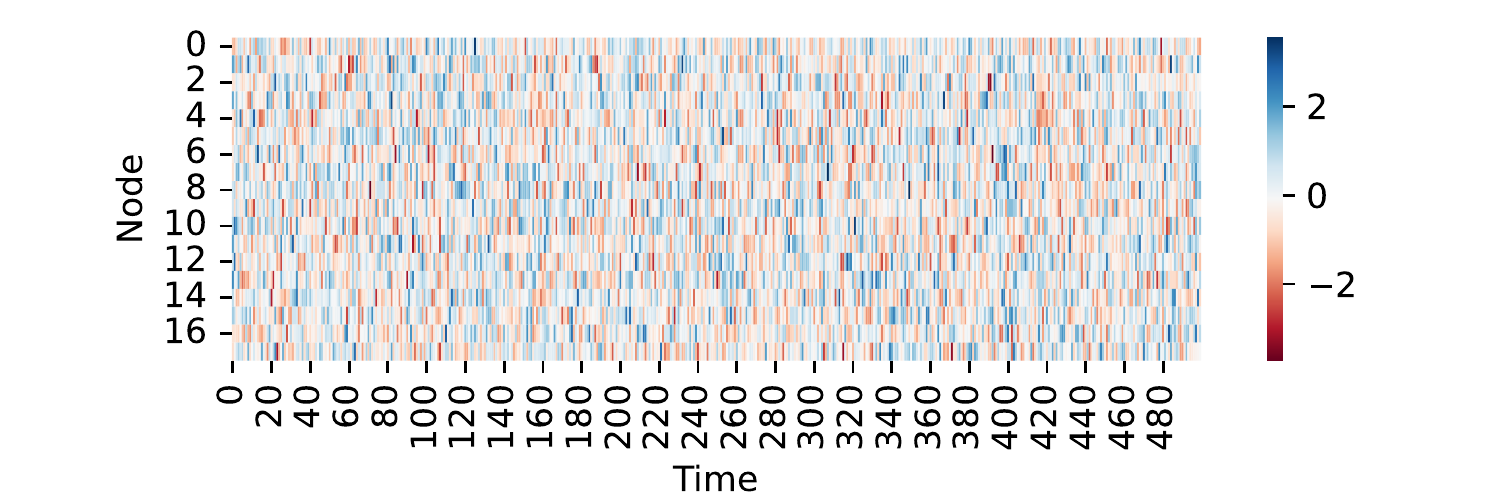}\\
\rotatebox{90}{\parbox{3cm}{\centering$c=0.16$}}
\includegraphics[height=3.2cm]{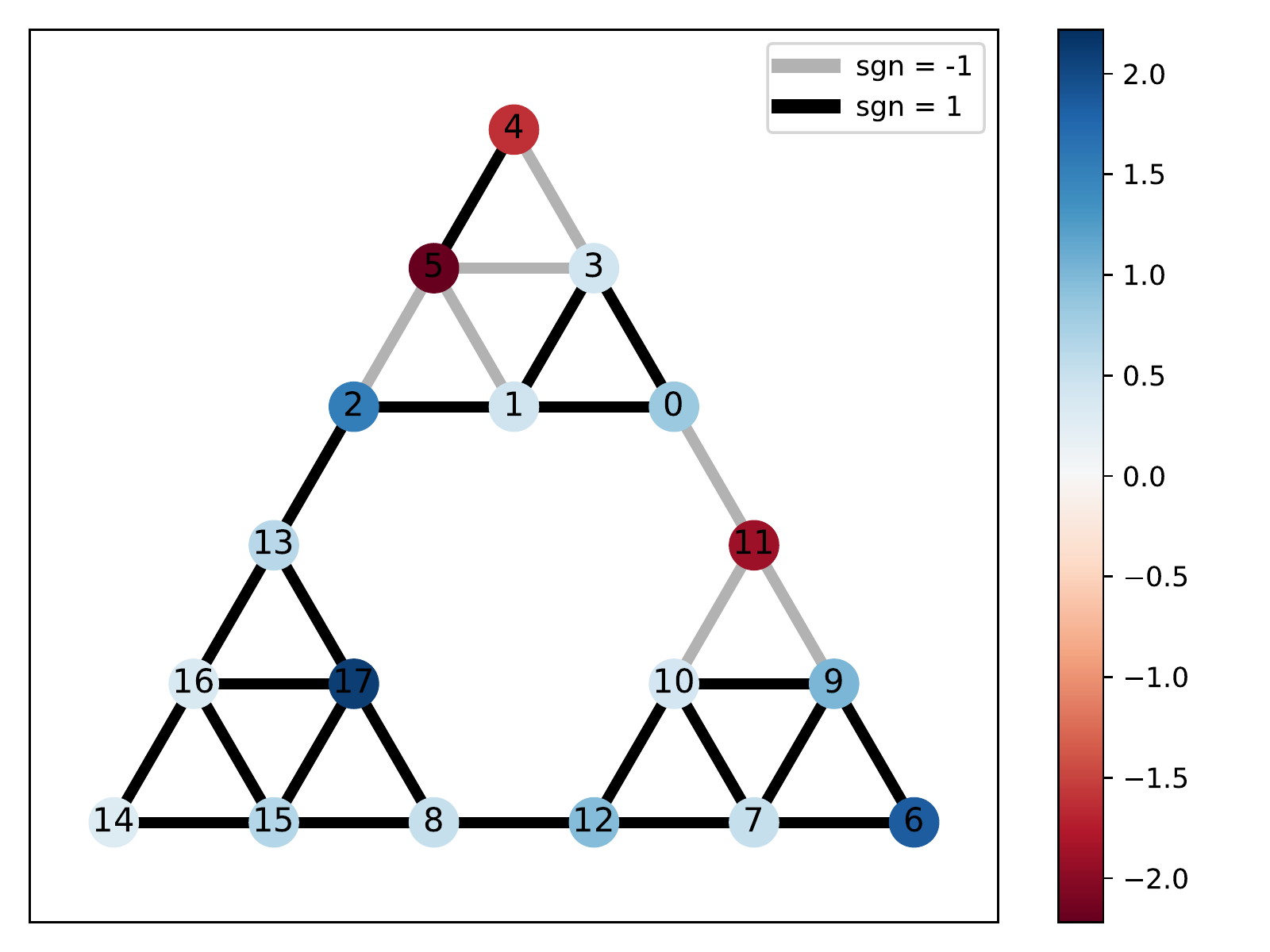}
\includegraphics[height=3.2cm]{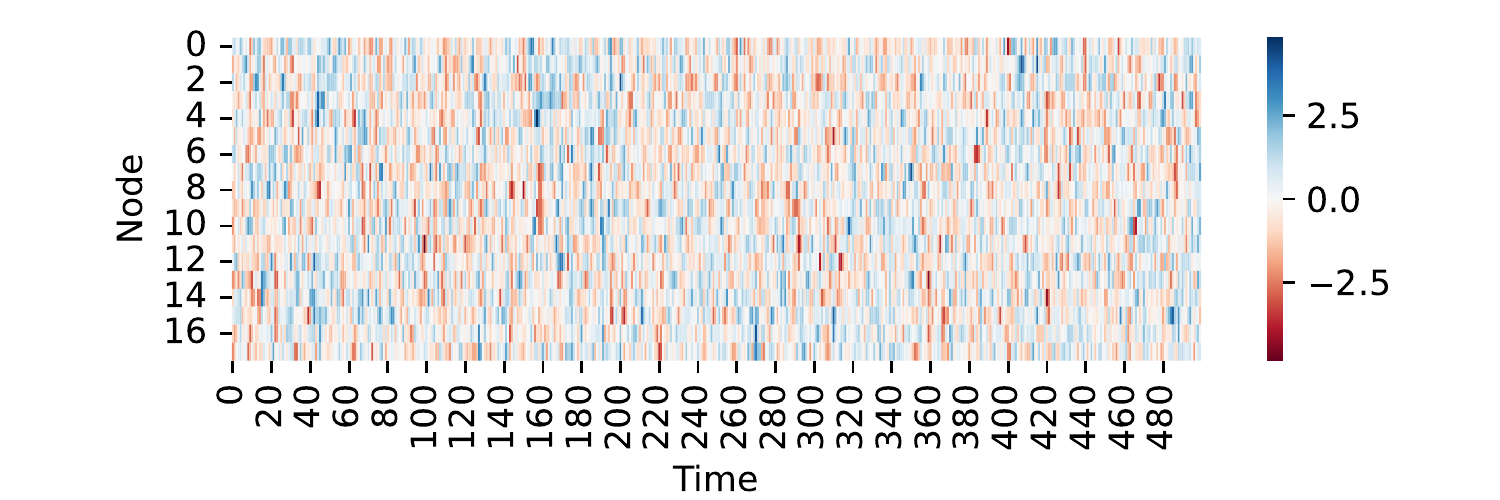}\\
\rotatebox{90}{\parbox{3cm}{\centering$c=0.64$}}
\includegraphics[height=3.2cm]{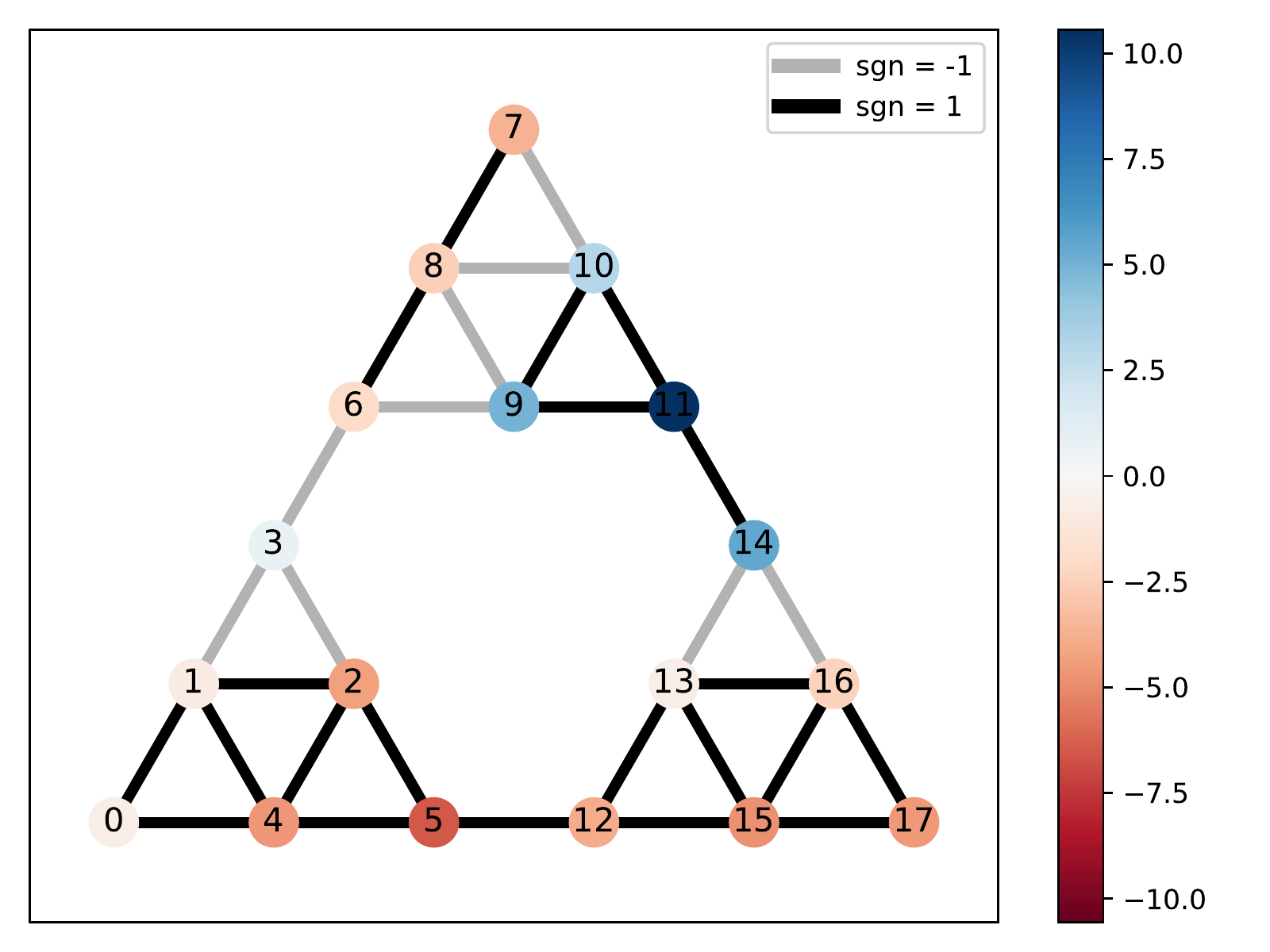}
\includegraphics[height=3.2cm]{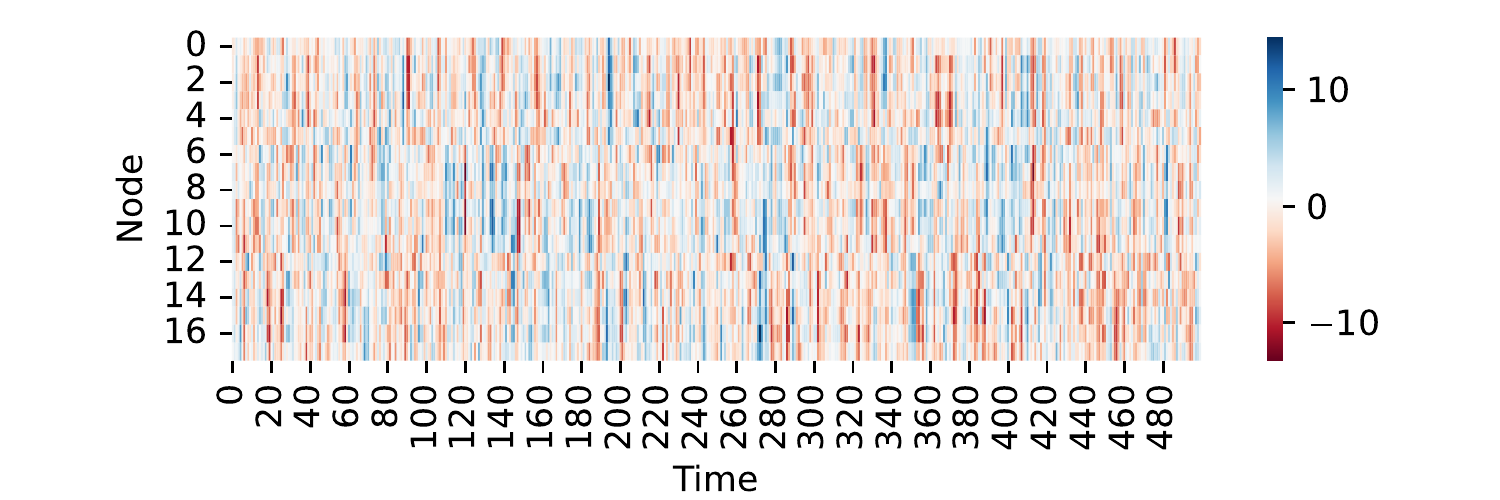}\\

\caption{Graph signals displaying different levels of correlation. Each row corresponds to a different value of $c_\tim=c_\spa=c$ in \eqref{eq:corr-signal-generation}. The number of time steps is $T=500$, the node feature dimension is $F=1$, and the distribution of all $\eta_v[t]$ is a standard Gaussian. On the right-hand side, we draw a heatmap of graph signals $\X$. On the left-hand side, we draw the underlying graph with nodes numbered according to the ordering used in the heatmaps. The node color encodes the value of the node signals at time $T/2$, and the edge color encodes the sign of the product of the corresponding node signals, \ie, $\sign(\x_u[T/2]\x_v[T/2])$ for all edges $(u,v)$.}
\label{fig:corr-signals}
\end{figure}

\begin{figure}
\centering
\includegraphics[width=\textwidth]{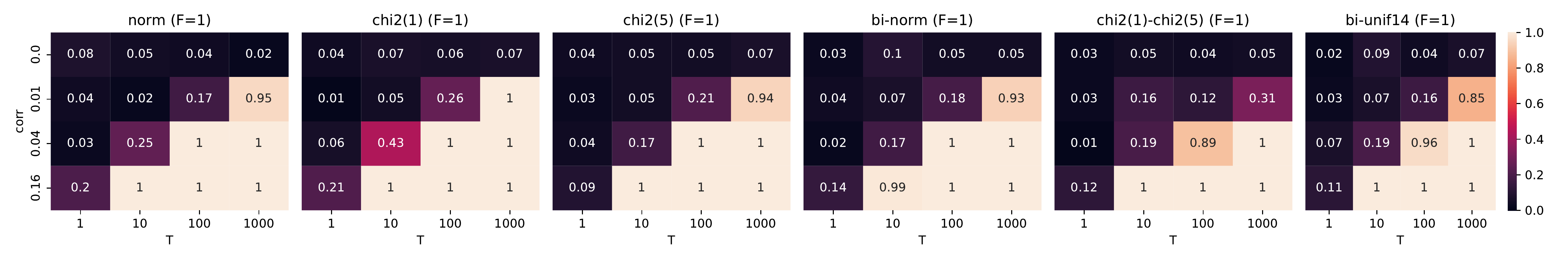}
\caption{Rate of rejected null hypotheses for significance level $\alpha=0.05$. Each block corresponds to a different distribution $P$. Every block has different correlation parameters $c$ on the rows and number of time steps $T$ on the columns. The node feature dimension is $F=1$. The underlying graph is that of Figure~\ref{fig:corr-signals}.}
\label{fig:power}
\end{figure}

\subsubsection{Results}

All experiments in the current section are repeated 100 times, each time generating a different sequence according to \eqref{eq:corr-signal-generation} and applying test \eqref{eq:test-temporal} with threshold $\gamma$ set to meet a user-defined significance level $\alpha=0.05$.  
When not specified, we consider test \eqref{eq:test-temporal}, \ie, statistic $C_{G_T}(\X_T;0.5)$  in \eqref{eq:test-stat-lambda}.
The considered figure of merit is the rate of rejected null hypotheses, that is, the number of times the test statistic is greater than $\gamma$ over the total number of repetitions of the same experiment.
This setting also applies to Figure~\ref{fig:sparse-vs-complete} discussed in Section~\ref{sec:sparse-vs-full}.

\paragraph{Power of the test.}
Figure~\ref{fig:power} shows the ability of the test to identify deviations from the null hypothesis for different values of parameter $c_\tim=c_\spa=c$, distribution $P$ and number of time steps $T$.
We observe that (i) the higher the value of $c$, the easier it is for the AZ-whiteness test to identify the correlation, and (ii) the test becomes more powerful as the number $T$ of time steps increases.
Similar results are displayed in Figure~\ref{fig:sparse-vs-complete}, where we varied the node feature dimension $F$.

\paragraph{Calibration of the test.}
From Figure~\ref{fig:power}, we can also analyze the correct calibration of the test in terms of desired significance level $\alpha$.  We can see that, regardless of the data distribution $P$ (symmetric/asymmetric, unimodal/bimodal), when the signal observations are independent --- \ie, $c=0$ --- the rate of rejected null hypotheses is around the predefined significance level $\alpha=0.05$, hence suggesting that the calibration is appropriate. For node feature dimensions $F>1$, we refer to Figure~\ref{fig:sparse-vs-complete} where, again, the observed rejection rates are close to the desired value $\alpha$.

\subsection{Optimality of forecasting models}
\label{sec:exp:optimality}

As a second set of experiments, we train state-of-the-art forecasting models on different synthetic and real-world graph signals, and we assess if they produce correlated residuals, hence implying that there is further information in the data that the models were not able to learn.

\begin{figure}
\includegraphics[width=0.4\textwidth]{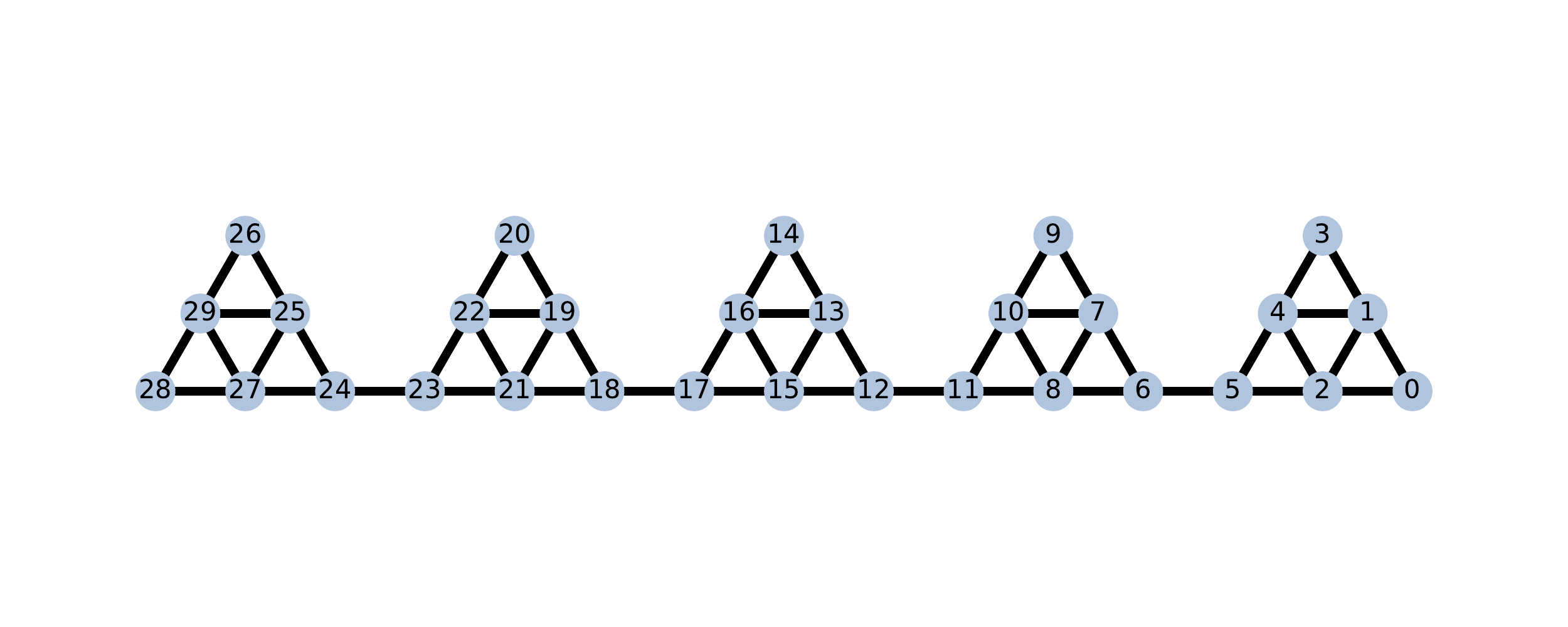}
\includegraphics[width=0.6\textwidth]{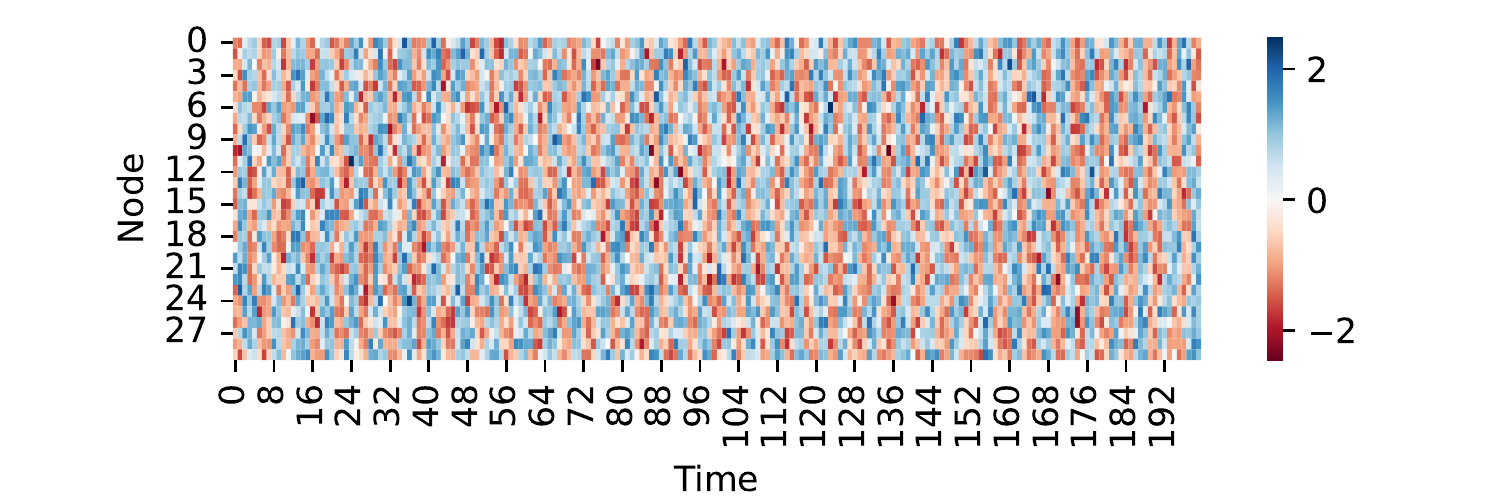}
\caption{Graph and graph signal from the GPVAR dataset of Section~\ref{sec:exp:optimality}. The displayed graph signal consists of the last $T=200$ consecutive time steps of the entire dataset.}
\label{fig:5-com-line-graph}
\end{figure}

\subsubsection{Experimental setting}

We consider three datasets: GPVAR, PemsBay, and MetrLA.
\begin{description}

    \item[GPVAR] 
    GPVAR is a synthetic dataset we generated from a graph polynomial vector autoregressive system model \cite{isufi2019forecasting}. The model generates each observation $\x_v[t]\in\R$ at time $t$ and node $v\in V$ from $Q\in \NN$ observations in the past and $L$-hop neighboring nodes, with $L\in\NN$, as follows:   
    \begin{equation}
    \label{eq:data-gen-gpvar}
    \x[t] = \tanh\left(\sum_{l=0}^L \sum_{q=1}^Q \Theta_{l,q} \vec S^l \x[t-p]\right) + \eta[t]
    \end{equation}
    where $\x[t]\in\R^N$ concatenates all scalar node signals at time step $t$, $\vec S$ is a graph shift operator, $\Theta\in \R^{(L+1)\times Q}$ collects the model parameters, and $\eta[t]\in\R^N$ is white noise generated form the standard Gaussian distribution [Section~\ref{sec:exp:correlation:data-gen}].
    In this paper, we consider the undirected and unweighted graph shown in Figure~\ref{fig:5-com-line-graph} with the following shift operator
    $
    \vec S = \vec D^{-1/2} (\vec I + \vec A) \vec D^{-1/2}
    $
    where  $\vec A$ is the adjacency matrix (with now self-loops), $\vec D$ is the diagonal degree matrix, and $\vec I $ the identity matrix.
    Model parameters $\Theta$ are set to $[[5, 2], [-4, 6], [-1, 0]]^\top$ and result in $L=Q=2$. 
    We generate $T=30000$ time steps.
    Figure~\ref{fig:5-com-line-graph} displays a graph signal generated according to \eqref{eq:data-gen-gpvar}.

    \item[PemsBay]
    PemsBay is a traffic dataset collected by the California Transportation Agencies Performance Measurement System (PeMS) \cite{li2018diffusion}. It presents $T=52128$ scalar observations from $N=325$ sensors in the Bay Area. 

    \item[MetrLA]    
    MetrLA is a dataset containing traffic information from $N=207$ detectors along the Los Angeles County highway \cite{li2018diffusion}. Data observations are collect for 4 months and amount to $T=34272$ time steps.

\end{description}

The graphs for PemsBay and MetrLA are constructed from the distances $\Delta_{u,v}$ between nodes
$u,v\in V$; not all pairwise distances are available.  
When $\Delta_{u,v}\in(0,\kappa)$, then the weighted adjacency matrix $\vec A_{u,v}=1$ and $w_{u,v}=e^{-\Delta_{u,v}^2/\sigma}$, otherwise is null; $\sigma$ is chosen to be the standard deviation of all distances in $(0, \kappa)$ \cite{li2018diffusion}.

The task to solve is a 1-step-ahead forecasting problem where, for every $t$, we want to predict graph signal $\x[t]$ from window $\vec X_{t-\tau:t-1}=\{\x[t-j]\in\R^N\st j=1,\dots,\tau\}$; $\tau$ is the size of the considered temporal window. 
In addition, for PemsBay and MetrLA, a positional encoding of the day of observation is added as exogenous variable. 
On the above forecasting tasks, we trained a Graph WaveNet \cite{wu2019graph} (GWNet), a Gated Graph Network \cite{satorras2022multivariate} (GatedGN), and a Diffusion Convolutional RNN \cite{li2018diffusion} (DCRNN). 
We also consider a baseline model (FCRNN) composed of a 1-layer fully-connected encoder $\vec z[t] = f(\x[t];\vec u[t])$ processing graph signal $\x[t]$ at each time step $t$ and the associated exogenous variables $\vec u[t]$, and a 2-layer GRU decoder applied to the resulting window of representations $[\vec z[t-\tau], \dots,\vec z[t-1]]$. 

Because our test assumes that the median of the residuals is null, we report the results when the predictions of the trained models are subtracted by their empirical median, and denote the corrected models with suffix ``-m'', \ie, FCRNN-m, GWaveNet-m, etc.

The size $\tau$ of the temporal window is set to 12 for all experiments. Train, validation and test sets contain 0.7, 0.1, and 0.2 of the original data. The models are trained until convergence with patience of $50$ epochs. 
All models and datasets are available in the TorchSpatiotemporal library \cite{TorchSpatiotemporal}, except for GPVAR dataset.  

In the following section, we report the mean absolute deviation (MAE) achieved by the above methods on the different datasets and the results of a statistical test on the median of the residuals. Metric MAE is, in fact, related to the median of the residuals in that the minimum of the MAE has to produce residuals with null median; see also Remark~\ref{remark:hp:median}. 
The test on the median is implemented as a test on the parameter of Bernoulli random variable $\sign(\x_v[t])$. Then, we report the results produced by the proposed whiteness test. In addition to the AZ-test run with parameter $\lambda=1/2$ [Equation~\ref{eq:test-stat-lambda}] accounting for both spatial and temporal correlations, we also considered $\lambda=0$ and $1$ to assess the impact of the temporal and spatial components alone, respectively.

\subsubsection{Results}

\begin{table}
\caption{Analysis of the observed residuals on GPVAR data. MAE is in the unit of measure of the graph signal, whereas $p$-values are reported as results of the statistical tests. AZ-tests with $p$-values larger than $0.01$ are represented in bold.}
\label{tab:gpvar}
\centering
\begin{tabular}{rccccc}
\toprule
\textbf{}  & MAE   & Median=0 & AZ-test($\lambda=0$) & AZ-test($\lambda=1/2$) & AZ-test($\lambda=1$) \\
\midrule
Optimal Predictor & 0.319 & 0.082  & \bfseries 0.354           & \bfseries 0.416                  & \bfseries 0.823 \\
\midrule
FCRNN      & 0.384 & 0.009    & $<$0.001             & $<$0.001               & \bfseries 0.066 \\  
FCRNM-m    & 0.384 & 1.000    & $<$0.001             & $<$0.001               & \bfseries 0.056 \\
GWNet      & 0.323 & $<$0.001 & \bfseries 0.705      &  \bfseries 0.709       &  \bfseries 0.881 \\
GWNet-m    & 0.323 & 1.000    & \bfseries 0.730      &  \bfseries 0.629       &  \bfseries 0.736 \\
GatedGN    & 0.321 & $<$0.001 &           0.005      &  \bfseries   0.171     &  \bfseries 0.414 \\
GatedGN-m  & 0.321 & 1.000    & \bfseries 0.015      &  \bfseries   0.257     &  \bfseries 0.411 \\
DCRNN      & 0.327 & $<$0.001 & \bfseries 0.533      &  \bfseries  0.955      &  \bfseries 0.587\\
DCRNN-m    & 0.327 & 1.000    & \bfseries 0.428      &  \bfseries  0.776      &  \bfseries 0.69\\
\bottomrule
\end{tabular}
\end{table}

We start by considering the synthetic dataset GPVAR. The results are reported in Table~\ref{tab:gpvar}. 
In this experiment, we know the data generating process \eqref{eq:data-gen-gpvar} and we are able to contrast the considered methods against the optimal predictor, \ie, the graph polynomial VAR filter with the same parameter $\Theta$ that generated the data.
From the first row of Table~\ref{tab:gpvar}, we observe that the whiteness test (AZ-test) for all values of $\lambda$ produced $p$-values significantly different from zero, suggesting uncorrelated residuals.
Baseline models FCRNN and FCRNN-m produce MAE substantially higher than the optimal MAE, regardless of the subtraction of the median. Consequently, we observe that AZ-test($\lambda=1/2$) suggests rejecting the null hypothesis. From the same test with $\lambda=0$ (only temporal edges are considered)  and $\lambda=1$  (only spatial edges), we see that the temporal correlation appears more prominent than the spatial one. 
In contrast, we can consider the training of GWNet, GatedGN, and DCRNN successful because they produce MAE close to the target value and relatively high $p$-values in the AZ-tests.

\begin{table}
\caption{Analysis of the observed residuals on PemsBay data. MAE is in the unit of measure of the graph signal, whereas $p$-values are reported as results of the statistical tests. AZ-tests with $p$-values larger than $0.01$ are represented in bold.}
\label{tab:pemsbay}
\centering
\begin{tabular}{rccccc}
\toprule
\textbf{}  & MAE       & Median=0 & AZ-test($\lambda=0$) & AZ-test($\lambda=1/2$) & AZ-test($\lambda=1$) \\
\midrule
FCRNN      & 2.015     & $<$0.001 & $<$0.001     & $<$0.001     & $<$0.001\\
FCRNN-m    & 2.015     & 0.993    & $<$0.001     & $<$0.001     & $<$0.001\\
GWNet      & 0.840     & $<$0.001 & $<$0.001     & $<$0.001     & $<$0.001\\
GWNet-m    & 0.840     & 0.990    & $<$0.001     & $<$0.001     & $<$0.001\\
GatedGN    & 0.838     & $<$0.001 & $<$0.001     & $<$0.001     & $<$0.001\\
GatedGN-m  & 0.838     & 0.990    & $<$0.001     & $<$0.001     & $<$0.001\\
DCRNN      & 0.845     & $<$0.001 & $<$0.001     & $<$0.001     & $<$0.001\\
DCRNN-m    & 0.845     & 0.988    & $<$0.001     & $<$0.001     & $<$0.001\\
\bottomrule
\end{tabular}
\end{table}

\begin{table}
\caption{Analysis of the residuals observed on MetrLA data. MAE is in the unit of measure of the graph signal, whereas $p$-values are reported as results of the statistical tests. AZ-tests with $p$-values larger than $0.01$ are represented in bold.}
\label{tab:metrla}
\centering
\begin{tabular}{rccccc}
\toprule
\textbf{}  & MAE       & Median=0 & AZ-test($\lambda=0$) & AZ-test($\lambda=1/2$) & AZ-test($\lambda=1$) \\
\midrule
FCRNN      & 2.841     & $<$0.001 & $<$0.001         & $<$0.001   & $<$0.001\\
FCRNN-m    & 2.841     & 0.999    & $<$0.001         & $<$0.001   & $<$0.001\\
GWNet      & 2.114     & $<$0.001 & $<$0.001         & $<$0.001   & $<$0.001\\
GWNet-m    & 2.114     & 1.000    & $<$0.001         & $<$0.001   & $<$0.001\\
GatedGN    & 2.150     & $<$0.001 & \bfseries 0.021  & $<$0.001   & $<$0.001\\
GatedGN-m  & 2.150     & 0.996    & \bfseries 0.016  & $<$0.001   & $<$0.001\\
DCRNN      & 2.141     & $<$0.001 & $<$0.001         & $<$0.001   & $<$0.001\\
DCRNN-m    & 2.140     & 1.000    & $<$0.001         & $<$0.001   & $<$0.001\\
\bottomrule
\end{tabular}
\end{table}

Regarding PemsBay and MetrLA datasets, we do not have a reference MAE, as the optimal predictor is unknown. 
From Tables~\ref{tab:pemsbay} and \ref{tab:metrla}, we observe that none of DCRNN, GatedGN, and GWNet (and their variants) produces uncorrelated residuals. 
However, in Table~\ref{tab:metrla} we see that the GatedGN produces residuals with a temporal correlation higher than the spatial one because when $\lambda=0$ the $p$-values are significantly larger than when $\lambda$ is $1$.

Finally, we point out that the null hypothesis is not rejected because Assumption \ref{hp:median} is not valid, in fact, in all cases, almost identical results are observed with the original methods (FCRNN, GWNet, GatedGN, DCRNN) and those with the median offset removed (FCRNN-m, GWNet-m, GatedGN-m, DCRNN-m).

\section{Conclusions}
\label{sec:conclusions}

In this work, we propose the first whiteness test for spatio-temporal time series defined over the nodes of a graph. 
By exploiting the known structural and functional relations defined by the underlying graph, we are able to identify both temporal and spatial correlations that exist among data observations.
We show that the test is asymptotically distribution-free with respect to the number of time steps and graph edges and, therefore, we can apply to data coming from arbitrary distributions and, even, nonidentically distributed observations.

The designed test is general, from which the definitive alphabetic encompassing name, and based on sign changes  between observations that are close in time or with respect to the graph connectivity.   
The test is very versatile, too, as it can exploit edge weights encoding the strength of the link,  
allows the graph topology to vary over time, and can operate when nodes are inserted or removed from the graph; these scenarios are frequent when dealing with real-world graph signals, like those coming from sensor networks, \eg, transportation networks, and smart grids.
Finally, the test is computationally scalable for sparse graphs with complexity linear in the number of edges and time steps. 

We empirically show that the proposed test is indeed capable of identifying dependencies among graph signals; when applied to analyze prediction residuals, the AZ-test can assess whether given forecasting models can be assumed to be optimal or not. 

The proposed whiteness test is pioneering in its nature, and allows for future extensions involving stochastic graphs and more sophisticated edge statistics.

\subsection*{Acknowledgements}
This work was supported by the Swiss National Science Foundation project FNS 204061: \emph{High-Order Relations and Dynamics in Graph Neural Networks.}

\bibliographystyle{abbrvnat}
\bibliography{pan}

\appendix 

\section{Proof of Theorem~\ref{theo:test-stat}}
\label{sec:proofs}

We start by proving two auxiliary lemmas, Lemma~\ref{lemma:sign-dist} and Lemma~\ref{lemma:sign-independency}, which will be used to prove Theorem~\ref{theo:test-stat}.

For brevity, in the following we use the short forms
\begin{align*}
s(e):=&s((u,v)):=\sign(\x_u^\top\x_v),
\\
\pi_v(\bar \x):=&\prob_{\x_v}(\bar \x^\top\x_v),
\\
\pi_e:=&\pi_{(u,v)}:=\EE_{\x_v}\left[\pi_u(\x_v)\right].
\end{align*}

\begin{lemma}
\label{lemma:sign-dist}
Under assumption \ref{hp:null} of independent $\x_v$ for $v\in V$, we have that
\begin{align*}
\prob(s(e)=a)&=\EE_{\x_v}[\pi_{u}(a\x_v)] = \EE_{\x_u}[\pi_{v}(a\x_u)] 
\\&= 
\begin{cases}
0 & a = 0,
\\
\pi_e & a > 0 ,
\\
1-\pi_e & a < 0.
\end{cases}
\end{align*}
Furthermore, if \ref{hp:median} holds, then $\EE[s(e)]=0$.
\end{lemma}

\begin{proof}
\begin{enumproof}
\item For any value $a\in \{-1, 0, 1\}$, edge $e=(u,v)\in E$, and distributions $P_v$ of all node signal $\x_v$ for $v\in V$, we have:
\begin{align*}
\prob(s(e)=a) &= \prob(a\,\x_v^\top\x_u>0) 
\\ 
    &= \int \prob(a\,\bar\x^\top\x_u>0|\x_v=\bar \x)\dint P_v(\bar \x) 
\\\ref{hp:null}\ 
    &= \int_{\R^F\setminus \{0\}} \prob(a\,\bar\x^\top\x_u>0)\dint P_v(\bar \x) 
\\
    &= \int \pi_u(a\,\bar\x)\dint P_v(\bar \x) = \EE_{\bar \x\sim P_v}[\pi_{u}(a\bar \x)] = \pi_e
\end{align*}
With analogous developments, we obtain $\prob(s(e)=a) = \EE_{\bar \x\sim P_u}[\pi_{v}(a\bar \x)]$, so $\pi_{e}=\pi_{(u,v)}=\pi_{(v,u)}$.

\item
Because for all $v$
$$\pi_v(\bar \x)=\prob_{\x_v}(\x_v^\top\bar \x>0) = 1 -\prob_{\x_v}(\x_v^\top\bar \x\le 0) = 1-\pi_v(-\bar \x),$$
we conclude that
\begin{align*}
\prob(s(e)=a) &= 
    \begin{cases}
    0 & a = 0,
    \\
    \pi_e & a > 0 ,
    \\
    1-\pi_e & a < 0,
    \end{cases}
\end{align*}
hence proving the first part of the thesis.

\item We observe that for all $u\in V,\bar \x\ne \vec 0$:
\begin{align*}
\EE_{\bar\x\sim P_v}[\sign(\bar\x^\top\x_u)]
&= 1\cdot\prob(\sign(\bar\x^\top\x_u)=1) + 0 + (-1)\prob(\sign(\bar\x^\top\x_u)=-1)
\\&= \pi_u(\bar \x) -(1-\pi_u(\bar \x))
= 2 \pi_u(\bar \x) - 1.
\end{align*}
Therefore, we have the following equivalence
\begin{equation}
\label{hp:median-pi}
\text{Assumption~\ref{hp:median} } \iff \pi_v(\bar\x) = \frac{1}{2} \text{ for all }v\in V,\;\bar\x\ne \vec 0.
\end{equation}
We conclude that, under \ref{hp:median},
\begin{align*}
\EE[s(e)]&= 1\cdot\pi_e + 0 + (-1)(1-\pi_e) = 2\pi_e -1 
\\ &= 2\EE_{\x_v}[\pi_u(\x_v)] -1 = 2\frac{1}{2}-1=0.
\end{align*}
\end{enumproof}
\end{proof}

\begin{lemma}
\label{lemma:sign-independency}
Consider two distinct edges $e,f\in E$ that are not self-loop.  
\begin{itemize}
    \item If \ref{hp:null} holds and $e,f$ are not adjacent, then random variables $s(e)$ and $s(f)$ are independent.
    \item Otherwise, if $e=(u,v),f=(u,w)$ for $v\ne w$, then $s(e)$ and $s(f)$ are independent if both \ref{hp:null} and \ref{hp:median} hold.  
\end{itemize}
\end{lemma}
\begin{proof}
To prove the thesis, we need to show that, for every $a,b\in \{-1,0,1\}$, 
$$
(*):=\prob(s(e)=a)\prob(s(f)=b) = \prob(s(e)=a, s(f)=b) =: (**)
$$

\begin{enumproof}
\item
If $e$ and $f$ are not adjacent, then $s(e),s(f)$ are independent because computed from different node pairs, which are independent by \ref{hp:null}. This proves the first part.

\item
Consider now $e=(u,v)$ and $f=(u,w)$ for some distinct $u, v, w\in V$.  
We start from $(*)$ and, in light of Lemma~\ref{lemma:sign-dist}, we have that
$$(*) = \EE_{\x_v}[\pi_{u}(a\x_v)] \cdot \EE_{\x_v}[\pi_{u}(b\x_w)].$$

\item
On the other side
\begin{align*}
(**) &=\prob(s(e)=a,s(f)=b)
\\
     &= \int \prob(a\,\bar\x^\top \x_v>0, b\,\bar\x^\top \x_w>0|\x_u=\bar \x) \dint P_{u}(\bar\x)
\\\ref{hp:null}\ 
     &= \int \prob(a\,\bar\x^\top \x_v>0)\prob(b\,\bar\x^\top \x_w>0) \dint P_{u}(\bar\x)
\\   &= \EE_{\x_u}\left[\pi_v(a\,\x_u)\pi_w(b\,\x_u)\right].
\end{align*}

\item For example, in the case of $P_v$ is the same for all $v\in V$, we reduce to Jensen's inequality, yielding $(*)=(**)$ only in particular cases, like when $\pi_v(\x)$ is affine on the support of $\x$. 
However, under the additional assumption of \ref{hp:median}, we have that $\pi_v(\x)=1/2$ for all $v$ and $\x$, so
$$
(*) = \begin{cases}
0 & ab = 0,
\\ 1/4  & ab \ne 0
\end{cases}
=(**).
$$
This proves the independency of $s(e)$ and $s(f)$.
\end{enumproof}
\end{proof}

Proven Lemmas~\ref{lemma:sign-dist} and \ref{lemma:sign-independency}, we return to Theorem~\ref{theo:test-stat}.
Without loss of generality, we can assume that the given graph $G=(V,E,\W)$ is undirected. Otherwise, we can always construct an undirected graph $\widetilde G=(\widetilde V,\widetilde E,\widetilde \W)$ from directed graph $G$ that produces the same test statistics $C_G(\X)=C_{\widetilde G}(\X)$ and for which $W_2$ in \eqref{eq:W2} simplifies to $\sum_{e\in E} w_e^2$;
to do so, it is enough to consider $\widetilde V=V$, construct $\widetilde E_u$ by removing the orientation of the edges in $E$, 
and considering $\widetilde w_{u,v}\in\widetilde \W$ as $w_{u,v}+w_{v,u}$ if both $(u,v)$ and $(v,u)\in E$, otherwise $\widetilde w_{u,v}$ equals either $w_{u,v}$ or $w_{v,u}$, depending which edge is present in $E$. Moreover, we note that the concept of correlation is not directional.

Under the assumption of undirected graph with no self-loops, we see that $\widetilde C_G(\X)$ in \eqref{eq:Ctilde} is a weighted sum of \iid Bernoulli random variables:
\begin{align*}
\widetilde C_G(\X) &= \sum_{(e) \in E} w_e s(e)
\end{align*}
or, equivalently, a sum of independent [Lemma~\ref{lemma:sign-independency}], but not identically distributed terms. Nevertheless, a central limit theorem \cite[Th.~27.3]{billingsley1995probability} can still be applied to show the thesis if the following Lindeberg condition on random variables $\{w_e\,s(e)\st e\in E\}$ holds.

Consider sequence of independent processes 
$$\big\{\{X_{n,k}\st k\le n\}\st n\in\NN\big\}$$
with $\EE[X_{n,k}]=0$ and $\var[X_{n,k}]=\sigma_{n,k}^2$ for all $k\le n,n\in\NN$, and 
$$s_n=\left(\sum_k \sigma_{n,k}^2\right)^{1/2}.$$
The Lindeberg condition \cite[Eq.~27.16]{billingsley1995probability} is the following
\begin{equation}
\label{eq:lindeberg-condition}
\lim_{n\to\infty} \frac{1}{s_n^2} \sum_{k=1}^n \int_{|X_{|E|,e}|\ge \varepsilon s_n} \lvert X_{n,k}\rvert^2 \dint P_{n,k} = 0, \quad\text{ for all } \varepsilon>0,
\end{equation}
where $P_{n,k}$ is the distribution of $X_{n,k}$.
We now show that \eqref{eq:lindeberg-condition} holds for $X_{|E|,e}=w_es(e)$. 

\begin{enumproof}
\item
Under assumptions~\ref{hp:null} and \ref{hp:median}, by Lemma~\ref{lemma:sign-independency} $\{X_{|E|,e}\st e\in E\}$ are independent and, by Lemma~\ref{lemma:sign-dist}, $\EE[X_{|E|,e}]=0$ with
$$
\EE[X_{|E|,e}^2]=w_e^2 \EE[s(e)^2] =w_e^2 \EE[1]= w_e^2.
$$
So $s_{|E|}^2=\sum_{e\in E} w_e^2\cdot 1 = W_2$. Substituting the quantities we have just calculated in \eqref{eq:lindeberg-condition}, we can express the Lindeberg condition as follows
\begin{align*}
     \frac{1}{s_{|E|}^{2}}\sum_{e\in E} &\int_{|X_{|E|,e}|\ge \varepsilon s_{|E|}} |w_e\,s(e)|^{2} \dint P_{s(e)} 
    = \frac{1}{s_{|E|}^{2}}\sum_{e\in E} \int_{|w_e||s(e)|\ge \varepsilon s_{|E|}} w_e^2 \dint P_{s(e)} 
\\ &= \frac{1}{s_{|E|}^{2}}\sum_{e\in E} w_e^2 I(|w_e|\ge \varepsilon s_{|E|})   
\end{align*}

\item
From \ref{hp:weights}, $s_{|E|}^2=W_2 \to \infty$  as $|E|\to\infty$. Therefore,
for every $\varepsilon>0$ there is a number $N(\varepsilon, w_+)$ of edges such that $w_+<\varepsilon s_{N(\varepsilon, w_+)}$, and for which $I(|w_e|\ge \varepsilon s_{N(\varepsilon,w_+)})=0$ for all $e$. We conclude that \eqref{eq:lindeberg-condition} holds true and, in turn, $C_G(\X)$ weakly converges to $\mc N(\EE[C_G(\X)],\var[C_G(\X)])$ by Theorem~27.3 in \cite{billingsley1995probability}.

\item
Finally, observe that the expected value of $C_G(\X)$ is
\begin{align*}
\EE[C_G(\X)]
   &= \sum_{e \in E} w_e \EE[s(e)] W_2^{-1/2} = 0
\end{align*}
thanks to Lemma~\ref{lemma:sign-dist}, and the variance is
\begin{align*}
\var[C_G(\X)] &= \frac{1}{W_2} \var\left[\sum_e w_e s(e)\right] 
\\\text{(Lem.~\ref{lemma:sign-independency}) }
&= \frac{1}{W_2} \sum_{e\in E} w_e^2 (\EE[s(e)^2]-\EE[s(e)]^2) 
\\
&= \frac{1}{W_2} W_2 (1 - 0) = 1.
\end{align*}
\end{enumproof}
We concluded the thesis of Theorem~\ref{theo:test-stat}.

\end{document}